\documentclass{article} 
\newif\iflongversion
\longversiontrue
\usepackage{nips14submit_e,times}
\usepackage{hyperref}
\usepackage{url}
\usepackage{multirow}
\usepackage{booktabs}
\usepackage{placeins}

\usepackage{tabularx}
\usepackage{algorithm}
\usepackage{algorithmic}
\usepackage{amsmath}
\usepackage{amssymb}
\usepackage{graphicx} 
\usepackage{epstopdf}
\usepackage{subfigure} 
\usepackage{bm}

\usepackage[numbers]{natbib}

\newcommand{\inner}[1]{\left\langle#1\right\rangle}

\def\R{\mathbb{R}}

\newcommand{\norm}[1]{\left\|#1\right\|}
\newcommand{\abs}[1]{\left|#1\right|}

\def\ones{\mathbf{1}}

\def\cut{\mathrm{cut}}
\def\TV{\mathrm{TV}}

\def\vol{\mathop{\rm vol}\nolimits}

\def\argmax{\mathop{\rm arg\,max}\limits}
\def\argmin{\mathop{\rm arg\,min}\limits}

\def\subj{\mathop{\rm subject\,to:}}

\def\ones{\mathbf{1}}

\def\cut{\mathrm{cut}}

\def\opt{\mathrm{*}}

\newtheorem{theorem}{Theorem}
\newtheorem{lemma}{Lemma}

\newtheorem{definition}{Definition}
\newtheorem{corollary}{Corollary}
\newtheorem{proposition}{Proposition}

\newenvironment{proof}{\par\noindent{\bf Proof:\ }}{\hfill$\Box$\\[2mm]}

\def\Oc{\mathcal{O}}
\def\BCut{\mathrm{BCut}}
\def\simplexCnstrs{simplex constraints}
\def\sizeCnstrs{size constraints}
\def\membershipCnstrs{membership constraints}
\def\descentCnstrs{descent constraints}
\def\labelCnstrs{label constraints}

\def\taubm{\bm{\tau}}
\def\sigmabm{\bm{\sigma}}
\title{Tight Continuous Relaxation of the Balanced $k$-Cut Problem}

\author{
Syama Sundar Rangapuram, Pramod Kaushik Mudrakarta and Matthias Hein \\
Department of Mathematics and Computer Science\\
Saarland University,
Saarbr\"ucken \\
}

\nipsfinalcopy 

\begin{document}

\maketitle

\begin{abstract}
Spectral Clustering as a relaxation of the normalized/ratio cut has become one of the standard graph-based clustering methods.
Existing methods for the computation of multiple clusters, corresponding to a balanced $k$-cut of the graph, are either based on greedy techniques or heuristics which have weak connection to the original motivation of minimizing the normalized cut. In this paper we propose a new tight continuous relaxation for any balanced $k$-cut problem and show that a related recently proposed relaxation is in most cases loose leading to poor performance in practice. For the optimization of our tight continuous relaxation we propose a new algorithm for the difficult sum-of-ratios minimization problem which achieves monotonic descent. Extensive comparisons show that our method outperforms all existing approaches for ratio cut and other balanced $k$-cut criteria.

\end{abstract}

%
%
%

\section{Introduction}
Graph-based techniques for clustering have become very popular in machine learning as they allow for an easy integration of pairwise relationships in data. The problem of finding $k$ clusters in a graph can be formulated as a balanced $k$-cut problem \cite{DonHof1973,PotSimLio1990,HagKah91,ShiMal2000}, where ratio and normalized cut are famous instances of balanced graph cut criteria
employed for clustering, community detection and image segmentation.  The balanced $k$-cut problem is known to be NP-hard \cite{ShiMal2000} and thus in practice relaxations \cite{ShiMal2000,NgJorWei2001} or greedy approaches \cite{DhiGuaKul2007} are used for finding the optimal multi-cut. The most famous approach is spectral clustering \cite{Lux07}, which
corresponds to the spectral relaxation of the ratio/normalized cut and uses $k$-means in the embedding of the vertices found by the first $k$ eigenvectors of the graph Laplacian in order to obtain the clustering. However, the spectral relaxation has been shown to be loose for $k=2$ \cite{GuaMil1998} and for $k>2$ no guarantees are known of the quality of the obtained $k$-cut with respect to the optimal one. Moreover, in practice even greedy approaches \cite{DhiGuaKul2007} frequently outperform spectral clustering.

This paper is motivated by another line of recent work \cite{SB10,HeiBue2010,HeiSet2011,BreLauUmiBre2012} where it has been shown that an exact continuous relaxation for the two cluster case ($k=2$) is possible for a quite general class of balancing functions. Moreover, efficient algorithms
for its optimization have been proposed which produce much better cuts than the standard spectral relaxation. However, the multi-cut
problem has still to be solved via the greedy recursive splitting technique.

Inspired by the recent approach in \cite{BreLauUmiBre2013}, in this paper we tackle directly the general balanced $k$-cut problem based on a new tight continuous relaxation. We show that the relaxation for the asymmetric ratio Cheeger cut proposed recently by \cite{BreLauUmiBre2013} is loose when the data does not contain $k$ well-separated clusters and thus leads to poor performance
in practice. Similar to \cite{BreLauUmiBre2013} we can also integrate label information leading to a transductive clustering formulation. Moreover, we propose an efficient algorithm for the minimization of our continuous relaxation for which we can prove monotonic descent. This is in contrast to the algorithm proposed in \cite{BreLauUmiBre2013} for which no such guarantee holds.
In extensive experiments we show that our method outperforms all existing methods in terms of the achieved balanced $k$-cuts. Moreover, our clustering error is competitive with respect to several other clustering techniques based on balanced $k$-cuts and
recently proposed approaches based on non-negative matrix factorization. Also we observe that already with small amount of label information the clustering error improves significantly.

%
%

\section{Balanced Graph Cuts}


Graphs are used in machine learning typically as similarity graphs, that is the weight of an edge between two instances encodes their
similarity. Given such a similarity graph of the instances, the clustering problem into $k$ sets can be transformed into a graph partitioning problem, where the goal is to construct a partition of the graph into $k$ sets such that the cut, that is the sum of weights of the edge from each set to all
other sets, is small and all sets in the partition are roughly of equal size.

Before we introduce balanced graph cuts, we briefly fix the setting and notation. Let $G(V,W)$ denote an undirected, weighted graph with vertex set $V$ with $n=|V|$ vertices and weight matrix $W \in \R_+^{n \times n}$ with $W=W^T$. There is an edge between two vertices $i,j \in V$ if $w_{ij}>0$.
The cut between two sets $A,B \subset V$ is defined as $\cut(A,B)=\sum_{i \in A, j\in B} w_{ij}$ and we write $\ones_A$ for the indicator vector of set $A\subset V$. A collection of $k$ sets $(C_1,\ldots,C_k)$ is a partition of $V$ if $\cup_{i=1}^k C_i=V$, $C_i \cap C_j =\emptyset$ if $i\neq j$ and $|C_i|\geq 1$, $i=1,\ldots,k$. We denote the set of all $k$-partitions of $V$ by $P_k$.
Furthermore, we denote by $\Delta_k$ the simplex $\{x : x\in \R^k,\ x\ge 0,\ \sum_{i=1}^{k} x_i= 1\}$.

Finally, a set function $\hat{S}:2^V \rightarrow \R$ is called submodular if for all $A,B\subset V$, $\hat{S}(A \cup B) + \hat{S}(A \cap B) \leq \hat{S}(A)+\hat{S}(B)$. Furthermore, we need the concept of the Lovasz extension of a set function.
\begin{definition}\label{def:Lovasz}
Let $\hat{S}:2^V \rightarrow \R$ be a set function with $\hat{S}(\emptyset)=0$.
Let $f \in \R^V$ be ordered in increasing order $f_1\leq f_2 \leq \ldots  \leq f_n$
and define $C_i = \{ j  \in V \, | \, f_j > f_i\}$ where $C_0=V$.
Then $S:\R^V \rightarrow \R$ given by, 
$S(f) \,=\, \sum_{i=1}^{n} f_i \Big(\hat{S}(C_{i-1}) - \hat{S}(C_i)\Big)$, 
 is called the \textbf{Lovasz extension} of $\hat{S}$. Note that $S(\ones_A)=\hat{S}(A)$ for all $A \subset V$.
\end{definition} 
The Lovasz extension of a set function is convex if and only if the set function is submodular \cite{Bac2013}.
The cut function $\cut(C,\overline{C})$, where $\overline{C}=V \backslash C$, is submodular and its Lovasz extension is given by $\TV(f)=\frac{1}{2}\sum_{i,j=1}^n w_{ij}|f_i-f_j|$.

\subsection{Balanced $k$-cuts}
The balanced $k$-cut problem is defined as
\begin{equation}\label{eq:setProb}
\min_{(C_1,\ldots,C_k) \in P_k} \quad \sum_{i=1}^k \frac{\cut(C_i,\overline{C_i})}{\hat{S}(C_i)}\; =:\; \BCut(C_1, \ldots, C_k)
\end{equation}
where $\hat{S}:2^V \rightarrow \R_+$ is a balancing function with the goal that all sets $C_i$ are of the same ``size''.
In this paper, we assume that $\hat{S}(\emptyset)=0$ and for any $C \subsetneq V,\ C \neq \emptyset$, $\hat{S}(C) \ge m$, for some $m>0$.
In the literature one finds mainly the following submodular balancing functions (in brackets is the name of the overall balanced graph cut criterion $\BCut(C_1,\ldots,C_k)$),
\begin{align} \label{eq:balancefct}
   \hat{S}(C)&=|C|, \quad  &(\textrm{Ratio Cut}), \\
   \hat{S}(C)&=\min\{|C|,|\overline{C}|\}, \quad &(\textrm{Ratio Cheeger Cut}),\nonumber \\
   \hat{S}(C)&=\min\{(k-1)|C|,\overline{C}\} \quad &(\textrm{Asymmetric Ratio Cheeger Cut}).\nonumber
\end{align}
The \emph{Ratio Cut} is well studied in the literature e.g. \cite{HagKah91,Lux07,DhiGuaKul2007} and corresponds to a balancing
function without bias towards a particular size of the sets, whereas the \emph{Asymmetric Ratio Cheeger Cut} recently proposed in \cite{BreLauUmiBre2013}
has a bias towards sets of size $\frac{|V|}{k}$ ($\hat{S}(C)$ attains its maximum at this point) which makes perfect sense if one expects clusters which have
roughly equal size. An intermediate version between the two is the \emph{Ratio Cheeger Cut} which has a symmetric balancing function and 
strongly penalizes overly large clusters. For the ease of presentation we restrict ourselves
to these balancing functions. However, we can also handle the corresponding weighted cases e.g., $\hat{S}(C)=\vol(C)=\sum_{i \in C}d_i$, where $d_i=\sum_{j=1}^n w_{ij}$, leading
to the \emph{normalized cut}\cite{ShiMal2000}. 

\section{Tight Continuous Relaxation for the Balanced $k$-Cut Problem}\label{sec:formulation}

In this section we discuss our proposed relaxation for the balanced $k$-cut problem \eqref{eq:setProb}. It turns out that a crucial question towards a tight multi-cut relaxation is the choice of the constraints so that the continuous problem also yields a partition (together with a suitable rounding scheme). The motivation for our relaxation is taken from the recent work of
\cite{SB10,HeiBue2010,HeiSet2011}, where exact relaxations are shown for the case $k=2$. Basically, they replace the ratio of set functions with the ratio of the
corresponding Lovasz extensions. We use the same idea for the objective of our continuous relaxation of the $k$-cut problem \eqref{eq:setProb} which is given as
\begin{alignat}{3}{\label{eq:contRel}}
	\min_{\substack{F = (F_1, \ldots, F_k),\\ F\in \R_+^{n \times k} }} &\; \sum_{l=1}^k 
							\frac{ \TV(F_l) } { S(F_l) } \\
	\subj 
			 &\; F_{(i)} \in \Delta_k	,\quad &&  i = 1, \ldots, n, \quad &&\textrm{(\simplexCnstrs)}\nonumber \\
			 &\; \max\{F_{(i)}\} = 1,\quad && \forall i \in I, \quad && \textrm{(\membershipCnstrs)} \nonumber\\
			 &\; S(F_l) \ge m,\quad && l = 1, \ldots, k, \quad && \textrm{(\sizeCnstrs)} \nonumber			 
\end{alignat}	
where $S$ is the Lovasz extension of the set function $\hat{S}$ and $m = \min_{C \subsetneq V,\ C \neq \emptyset} \hat{S}(C)$. 
We have $m=1$, for \emph{Ratio Cut} and \emph{Ratio Cheeger Cut} whereas $m=k-1$ for \emph{Asymmetric Ratio Cheeger Cut}. Note that $\TV$ is the Lovasz extension of the cut functional $\cut(C, \overline{C})$. In order to simplify notation we denote for a matrix $F \in \R^{n \times k}$ by $F_l$ the $l$-th column
of $F$ and by $F_{(i)}$ the $i$-th row of $F$. Note that the rows of $F$ correspond to the vertices of the graph and the $j$-th column of $F$ corresponds to the set $C_j$ of the desired partition. 
The set $I \subset V$ in the membership constraints is chosen 
adaptively by our method during the sequential optimization described in Section \ref{sec:alg}. 

An obvious question is how to get from the continuous solution $F^*$ of \eqref{eq:contRel} to a partition $(C_1,\ldots,C_k) \in P_k$ which is typically called \emph{rounding}. Given $F^*$ we construct the sets, by assigning each vertex $i$ to the column where the $i$-th row attains its maximum. Formally,
\begin{equation}\label{eq:rounding}
 C_i = \{j \in V\,|\, i=\argmax_{s=1,\ldots,k} F_{js}\}, \quad i=1, \ldots, k, \quad \textrm{(Rounding)}
\end{equation}
where ties are broken randomly. If there exists a row such that the rounding is not unique, we say that the solution is weakly degenerated. If furthermore
the resulting set $(C_1,\ldots,C_k)$ do not form a partition, that is one of the sets is empty, then we say that the solution is strongly degenerated.

First, we connect our relaxation to the previous work of \cite{HeiSet2011} for the case $k=2$. Indeed for symmetric
balancing function such as the \emph{Ratio Cheeger Cut}, our continuous relaxation \eqref{eq:contRel} is exact even without
membership and size constraints. 
\begin{theorem}
Let $\hat{S}$ be a non-negative symmetric balancing function, $\hat{S}(C)=\hat{S}(\overline{C})$, and denote by $p^*$
the optimal value of \eqref{eq:contRel} without membership and size constraints  for $k=2$. Then it holds
\[ p^* = \min_{(C_1,C_2) \in P_2} \sum_{i=1}^2 \frac{\cut(C_i,\overline{C_i})}{\hat{S}(C_i)}.\]
Furthermore there exists a solution $F^*$ of \eqref{eq:contRel} such that $F^*=[\ones_{C^*},\ones_{\overline{C^*}}]$, where
$(C^*,\overline{C^*})$ is the optimal balanced $2$-cut partition.
\end{theorem}
\iflongversion
\begin{proof}
Note that $\cut(C,\overline{C})$ is a symmetric set function and $\hat{S}$ by assumption. Thus with $C_2=\overline{C_1}$, 
\[  \frac{\cut(C_1,\overline{C_1})}{\hat{S}(C_1)} + \frac{\cut(C_2,\overline{C_2})}{\hat{S}(C_2)} = 2\frac{\cut(C_1,\overline{C_1})}{\hat{S}(C_1)}.\]
Moreover, as $\TV(\alpha f+\beta \ones)=\abs{\alpha} \,TV(f)$ and by symmetry of $\hat{S}$ also $S(\alpha f+\beta \ones)=\abs{\alpha} \,S(f)$ (see \cite{Bac2013,HeiSet2011}). The simplex constraint implies that $F_2=\ones-F_1$ and thus
\[ \frac{ \TV(F_2) } { S(F_2) } =  \frac{ \TV(\ones-F_1) } { S(\ones-F_1) } = \frac{\TV(F_1)}{S(F_1)}.\]
Thus we can write problem \eqref{eq:contRel} equivalently as 
\[ \min_{f \in [0,1]^V} 2\frac{\TV(f)}{S(f)}.\]
As for all $A \subset V$, $\TV(\ones_A)=\cut(A,\overline{A})$ and $S(\ones_A)=\hat{S}(A)$, we have 
\[ \min_{f \in [0,1]^V} \frac{\TV(f)}{S(f)} \leq \min_{C\subset V} \frac{\cut(C,\overline{C})}{\hat{S}(C)}.\]
However, it has been shown in \cite{HeiSet2011} that $\min_{f \in \R^V} \frac{\TV(f)}{S(f)} = \min_{C \subset V} \frac{\cut(C,\overline{C})}{\hat{S}(C)}$
and that there exists a continuous solution such that $f^*=\ones_{C^*}$, where $C^*=\argmin_{C \subset V} \frac{\cut(C,\overline{C})}{\hat{S}(C)}$. As $F^*=[f^*,\ones-f^*]=[\ones_{C^*},\ones_{\overline{C^*}}]$ this finishes the proof.
 \end{proof}
 \fi
Note that rounding trivially yields a solution in the setting of the previous theorem. 
 
A second result shows that indeed our proposed optimization problem \eqref{eq:contRel} is a relaxation
of the balanced $k$-cut problem \eqref{eq:setProb}. Furthermore, the relaxation is exact if $I=V$.
 \begin{proposition}\label{pro:keyResult}
 	The continuous problem \eqref{eq:contRel} is a relaxation of the $k$-cut problem \eqref{eq:setProb}. 
 	The relaxation is exact, i.e., both problems are equivalent, if $I=V$.
 \end{proposition} 	
\iflongversion
 \begin{proof}
 For any $k$-way partition $(C_1, \ldots, C_k)$, we can construct $F = (\ones_{C_1}, \ldots, \ones_{C_k})$. It obviously satisfies the membership and size constraints and the simplex constraint is satisfied as $\cup_i C_i=V$ and $C_i \cap C_j =\emptyset$ if $i\neq j$. Thus $F$ is feasible for problem \eqref{eq:contRel} and has the same objective value because 
 \[ \TV(\ones_C)=\cut(C,\overline{C}), \quad S(\ones_C)=\hat{S}(C).\]
Thus problem \eqref{eq:contRel} is a relaxation of \eqref{eq:setProb}.

If $I=V$, then the simplex together with the membership constraints imply that each row $F_{(i)}$ contains exactly one non-zero element which equals 1, i.e., $F \in \{0, 1\}^{n \times k}$.
Define for $l=1, \ldots, k$, $C_l=\{ i \in V\,|\, F_{il}=1\}$ (i.e, $F_l = \ones_{C_l}$), then it holds $\cup_l C_l =V$ and $C_l  \cap C_j=\emptyset$, $l\neq j$.
From the \sizeCnstrs, we have for $l=1, \ldots, k$, $0 < m \le S(F_l) = S(\ones_{C_l}) = \hat{S}(C_l)$. Thus $\hat{S}(C_l) >0,\ l=1, \ldots, k$, which by assumption on $\hat{S}$ implies that each $C_l$ is non-empty.
Hence the only feasible points allowed are indicators of $k$-way partitions and the equivalence of \eqref{eq:setProb} and \eqref{eq:contRel} follows.
 \end{proof}
 \fi
The row-wise simplex and membership constraints enforce that each vertex in $I$ belongs to exactly one component.
Note that these constraints alone (even if $I=V$) can still not guarantee that $F$ corresponds to a $k$-way partition since an entire column of $F$ can be zero.
This is avoided by the column-wise \sizeCnstrs\ that enforce that each component has at least one vertex. 

If $I=V$ it is immediate from the proof that problem \eqref{eq:contRel} is no longer a continuous problem as the feasible set are only indicator matrices of
partitions. In this case rounding yields trivially a partition. On the other hand, if $I=\emptyset$ (i.e., no membership constraints), and $k>2$ it is not guaranteed that rounding of the solution of the continuous problem yields a partition. Indeed, we will see in the following that for symmetric balancing functions one can, under these conditions,
show that the solution is always strongly degenerated and rounding does not yield a partition (see Theorem \ref{th:simplex}).
Thus we observe that the index set $I$ controls the degree to which the partition constraint is enforced. 
The idea behind our suggested relaxation is that it is well known in image processing that minimizing the total variation yields piecewise constant solutions
(in fact this follows from seeing the total variation as Lovasz extension of the cut). Thus if $|I|$ is sufficiently large, the vertices where the values are fixed to $0$ or $1$ propagate this to their neighboring vertices and finally to the whole graph. We discuss the choice of $I$ in more detail in
Section \ref{sec:alg}.

\paragraph{Simplex constraints alone are not sufficient to yield a partition:}
Our approach has been inspired by \cite{BreLauUmiBre2013} who proposed the following continuous relaxation for the \emph{Asymmetric Ratio Cheeger Cut} 
 \begin{alignat}{2}\label{eq:BreRel}
	\min_{\substack{F = (F_1, \ldots, F_k),\\ F\in \R_+^{n \times k} }} &\; \sum_{l=1}^k 
							\frac{ \TV(F_l) } { \norm{F_l - \textrm{quant}_{k-1} (F_l)}_1 } \\
	\subj 
			  &\; F_{(i)} \in \Delta_k,\quad   i = 1, \ldots, n\nonumber,\quad && \textrm{(\simplexCnstrs)}
\end{alignat}
where $S(f)=\norm{f - \textrm{quant}_{k-1} (f)}_1$ is the Lovasz extension of $\hat{S}(C)=\min\{(k-1)|C|,\overline{C}\}$ and $\textrm{quant}_{k-1}(f)$ is the $k-1$-quantile of $f \in \R^n$. Note that in their approach no membership constraints and size constraints are present.



We now show that the usage of simplex constraints in the optimization problem \eqref{eq:contRel} is not sufficient to guarantee that the solution $F^*$ can be rounded
to a partition for any symmetric balancing function in \eqref{eq:setProb}. For asymmetric balancing functions as employed for the \emph{Asymmetric Ratio Cheeger Cut} by \cite{BreLauUmiBre2013} in their relaxation \eqref{eq:BreRel} we can prove such a strong result only in the case where the graph is disconnected. However, note that if the number of components of the graph is less than the number of desired
clusters $k$, the multi-cut problem is still non-trivial. 

 \begin{theorem}\label{th:simplex}
	Let $\hat{S}(C)$ be any non-negative symmetric balancing function. Then the continuous relaxation 
	\begin{align}\label{eq:simplex}
	\min_{\substack{F = (F_1, \ldots, F_k),\\ F\in \R_+^{n \times k} }} &\; \sum_{l=1}^k 
							\frac{ \TV(F_l) } { S(F_l) } \\
	\subj &\; F_{(i)} \in \Delta_k,\quad i = 1, \ldots, n,\quad \textrm{(\simplexCnstrs)} \nonumber
	\end{align}
	of the balanced $k$-cut problem \eqref{eq:setProb} is void in the sense that the optimal solution $F^*$ of the 
	continuous problem can be constructed from the optimal solution of the $2$-cut problem 
	and $F^*$ cannot be rounded into a $k$-way partition, see \eqref{eq:rounding}.
	If the graph is disconnected, then the same holds also for any non-negative asymmetric balancing function.
	\end{theorem}
	\iflongversion
	\begin{proof}
 		First, we derive a lower bound on the optimum of the continuous relaxation \eqref{eq:simplex}. 
 		Then we construct a feasible point for \textbf{\eqref{eq:simplex} } that achieves this lower bound but cannot yield a partitioning thus finishing the proof.
	
		 Let $(C^\opt,\overline{C^\opt})=\argmin_{C \subset V} \frac{\cut(C,  \overline{C})}{\hat{S}(C)}$ be an optimal $2$-way partition for the given graph.
		 Using the exact relaxation result for the balanced $2$-cut problem in Theorem 3.1. in \cite{HeiSet2011}, we have 
\[ \min_{F : F_{(i)} \in \Delta_k } \sum_{l=1}^k \frac{\TV(F_l)}{S(F_l)} \ge \sum_{l=1}^k\min_{f \in \R^n} \frac{\TV(f)}{S(f)} = \sum_{l=1}^k \min_{C \subset V} \frac{\cut(C, \overline{C})}{\hat{S}(C)} = k \,\frac{\cut(C^\opt, \overline{C^\opt})}{\hat{S}(C^\opt)}.
\]
		 
 Now define $F_1 = \ones_{C^\opt}$ and $F_l = \alpha_l \ones_{\overline{C^\opt}}, \ l=2, \ldots, k$ such that 
 $\sum_{l=2}^k \alpha_l = 1, \alpha_l > 0$.
 Clearly $F = (F_1, \dots, F_k)$ is feasible for the problem \eqref{eq:simplex} and the corresponding objective value is
 \begin{align*}
 	\frac{\TV(\ones_{C^\opt})}{S(\ones_{C^\opt})} + \sum_{l=2}^k  
 		\frac{\alpha_l \TV(\ones_{\overline{C^\opt}})}  {\alpha_l S(\ones_{\overline{C^\opt}})}= \sum_{l=1}^k \frac{\cut(C^\opt, \overline{C^\opt})}{\hat{S}(C^\opt)},
\end{align*} 
where we used the $1$-homogeneity of $\TV$ and $S$ \cite{Bac2013} and the symmetry of $\cut$ and $\hat{S}$.

Thus the solution $F$ constructed as above from the $2$-cut 
problem is indeed optimal for the continuous relaxation \eqref{eq:simplex} and it is not possible to obtain a $k$-way partition from this solution
as there will be $k-2$ sets that are empty.
Finally, the argument can be extended to asymmetric set functions if there exists a set $C$ such that $\cut(C, \overline{C})=0$ as in this case it does not matter that $\hat{S}(C) \neq \hat{S}(\overline{C})$ in order that the argument holds.
	\end{proof}
\fi	
The proof of Theorem \ref{th:simplex} shows additionally that for any balancing function if the graph is disconnected, the solution of the continuous relaxation \eqref{eq:simplex} is always zero, while clearly the solution of the balanced $k$-cut problem need not be zero. This shows that the relaxation can be arbitrarily
bad in this case. 
In fact the relaxation for the asymmetric case can even fail if the graph is not disconnected but there exists a cut of the graph which is very small
as the following corollary indicates.
\begin{corollary}\label{co:loose}
Let $\hat{S}$ be an asymmetric balancing function and $C^*=\argmin_{C \subset V} \frac{\cut(C,\overline{C})}{\hat{S}(C)}$ and suppose that
$\phi^*:=(k-1) \frac{\cut(C^*,\overline{C^*})}{\hat{S}(C^*)} +  \frac{\cut(C^*,\overline{C^*})}{\hat{S}(\overline{C^*})} \; < \; \min_{(C_1,\ldots,C_k) \in P_k} \sum_{i=1}^k \frac{\cut(C_i,\overline{C_i})}{\hat{S}(C_i)}.$
Then there exists a feasible $F$ with  $F_1 = \ones_{\overline{C^\opt}}$ and $F_l = \alpha_l \ones_{C^\opt}, \ l=2, \ldots, k$  such that 
 $\sum_{l=2}^k \alpha_l = 1, \alpha_l > 0$ for \eqref{eq:simplex} which has objective
$\sum_{i=1}^k \frac{\TV(F_i)}{S(F_i)} = \phi^*$
and which cannot be rounded to a $k$-way partition.
\end{corollary}
	\iflongversion
\begin{proof}
Let $F_1 = \ones_{\overline{C^\opt}}$ and $F_l = \alpha_l \ones_{C^\opt}, \ l=2, \ldots, k$  such that 
 $\sum_{l=2}^k \alpha_l = 1, \alpha_l > 0$.   Clearly $F = (F_1, \dots, F_k)$ is feasible for the problem \eqref{eq:simplex} and the corresponding objective value is
 \begin{align*}
 \sum_{l=1}^k \frac{\TV(F_l)}{S(F_l)}& =\frac{\TV(\ones_{\overline{C^\opt}})}{S(\ones_{\overline{C^\opt}})} + \sum_{l=2}^k  
 		\frac{\alpha_l \TV(\ones_{C^\opt})}  {\alpha_l S(\ones_{C^\opt})}\\&= \frac{\cut(C^*,\overline{C^*})}{\hat{S}(\overline{C^*})} + (k-1) \frac{\cut(C^*,\overline{C^*})}{\hat{S}(C^*)},
\end{align*} 
where we used the $1$-homogeneity of $\TV$ and $S$ \cite{Bac2013} and the symmetry of $\cut$. 
This $F$ cannot be rounded into a $k$-way partition
as there will be $k-2$ sets that are empty.
\end{proof} 
\fi	
Theorem \ref{th:simplex} shows that the membership and size constraints which we have introduced in our relaxation \eqref{eq:contRel} are essential to obtain a partition for symmetric balancing functions. For the asymmetric balancing function
failure of the relaxation \eqref{eq:simplex} and thus also of the relaxation \eqref{eq:BreRel} of \cite{BreLauUmiBre2013} is only guaranteed for disconnected graphs. However, Corollary \ref{co:loose} indicates that degenerated solutions should also be a problem when the graph is still connected but there exists a dominating cut.
We illustrate this with a toy example in Figure \ref{fig:counterEX} where the algorithm of \cite{BreLauUmiBre2013} for solving \eqref{eq:BreRel} fails as it converges exactly to the solution predicted
by Corollary \ref{co:loose} and thus only produces a $2$-partition instead of the desired $3$-partition. The algorithm for our relaxation enforcing 
membership constraints converges to a continuous solution which is in fact a partition matrix so that no rounding is necessary.

	\begin{figure}
	\subfigure[]{
	\includegraphics[width=0.18\textwidth]{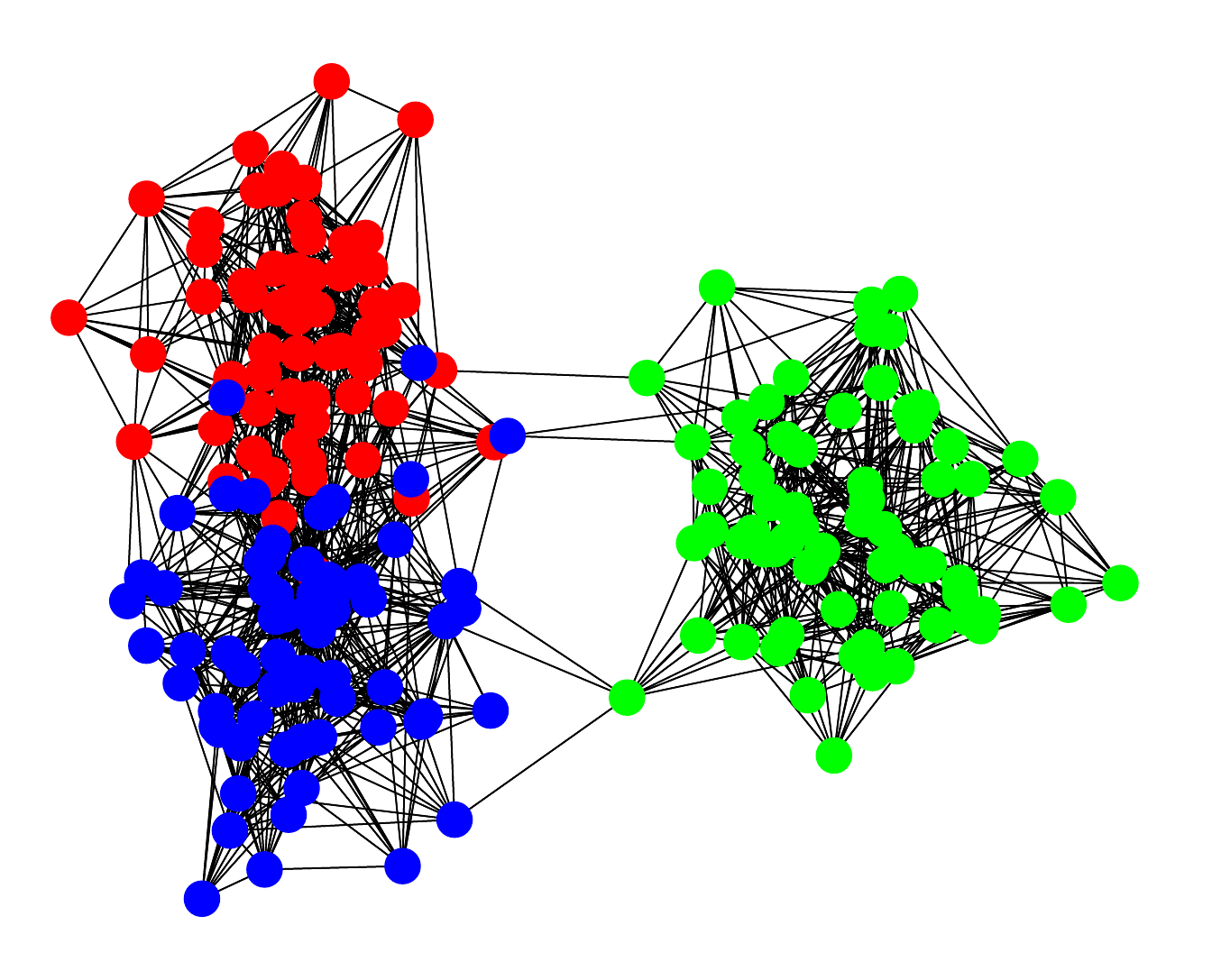}}
	\subfigure[]{
	\includegraphics[width=0.19\textwidth]{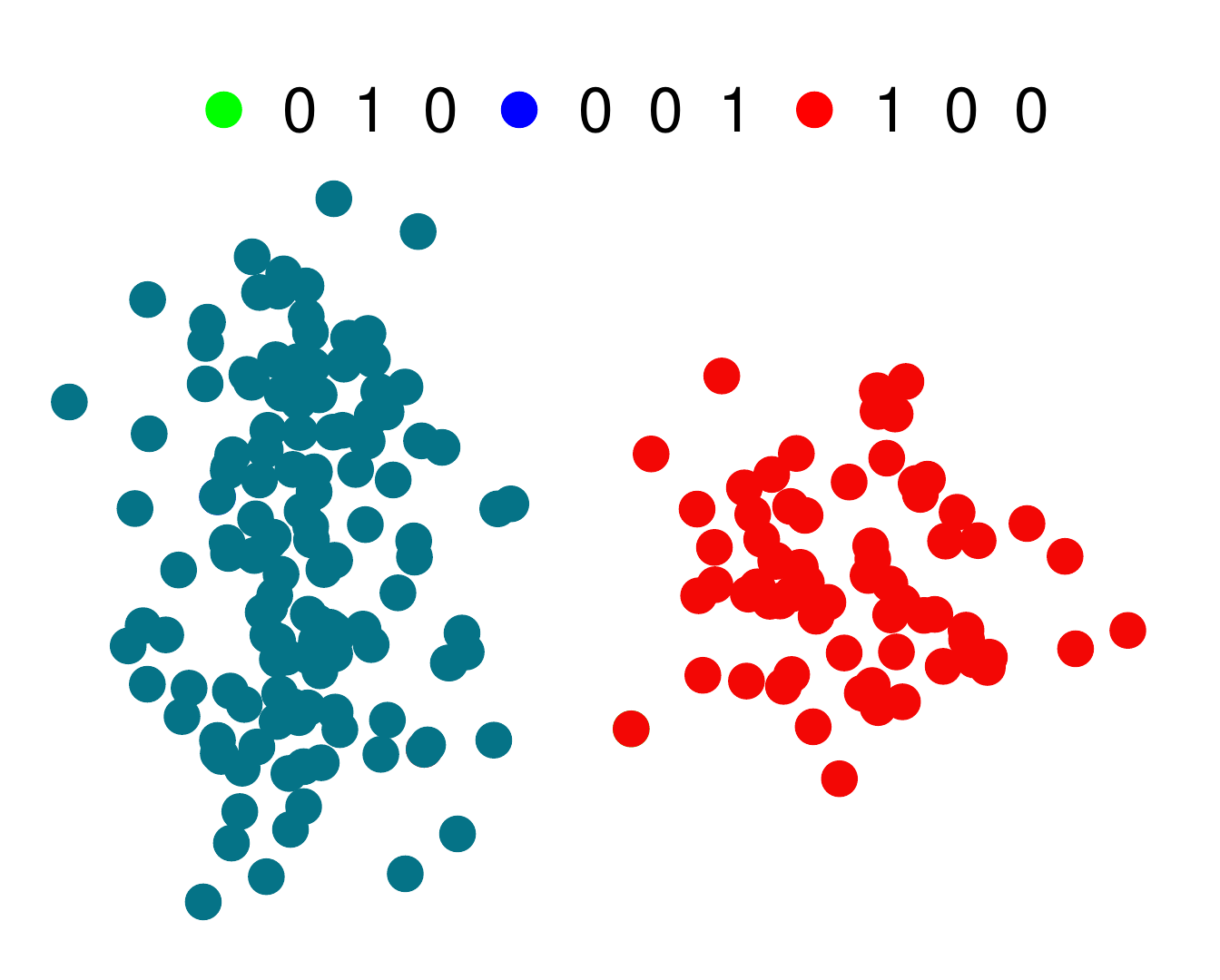}}
	\subfigure[]{
	  \includegraphics[width=0.19\textwidth]{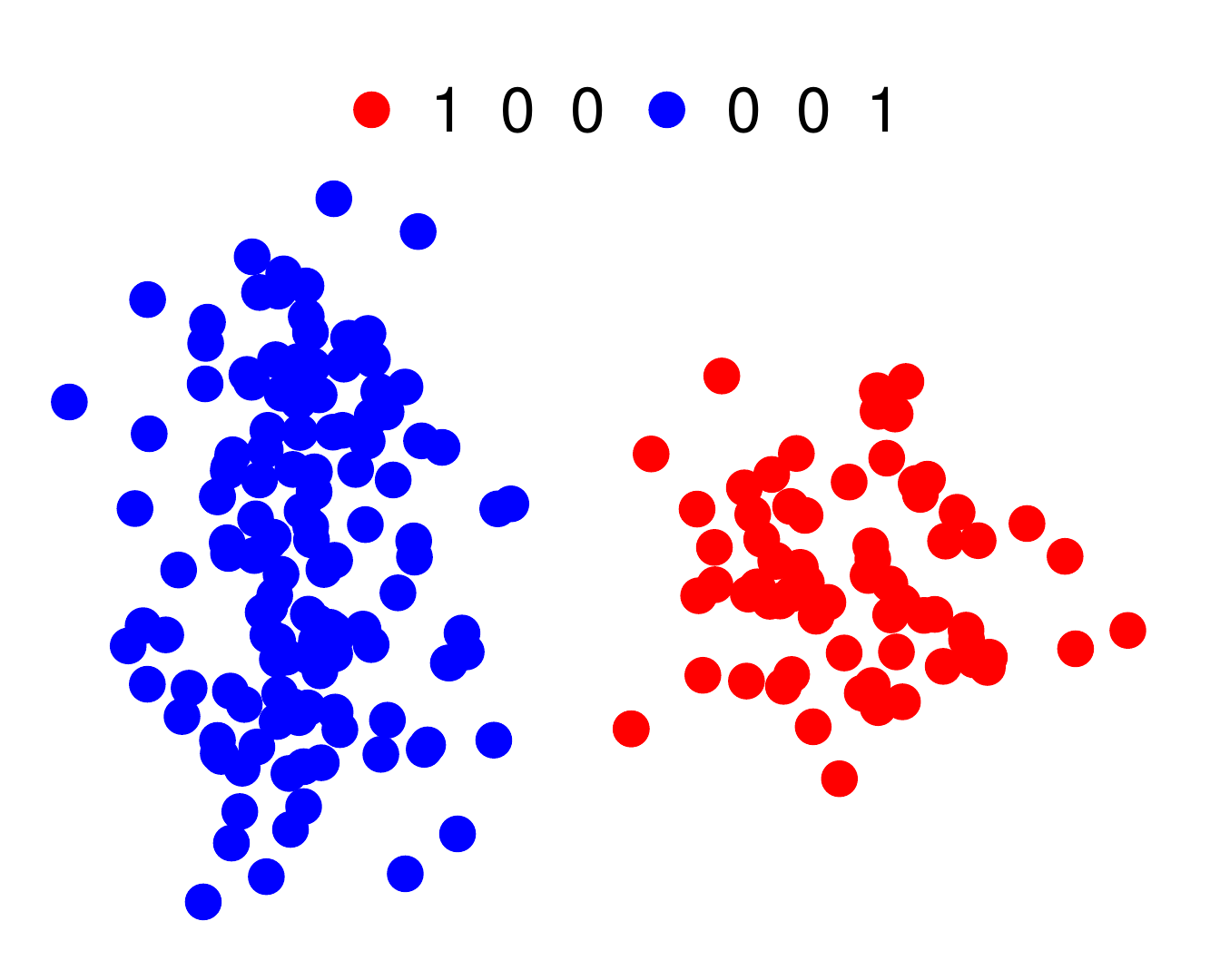}}
	\subfigure[]{
   \includegraphics[width=0.19\textwidth]{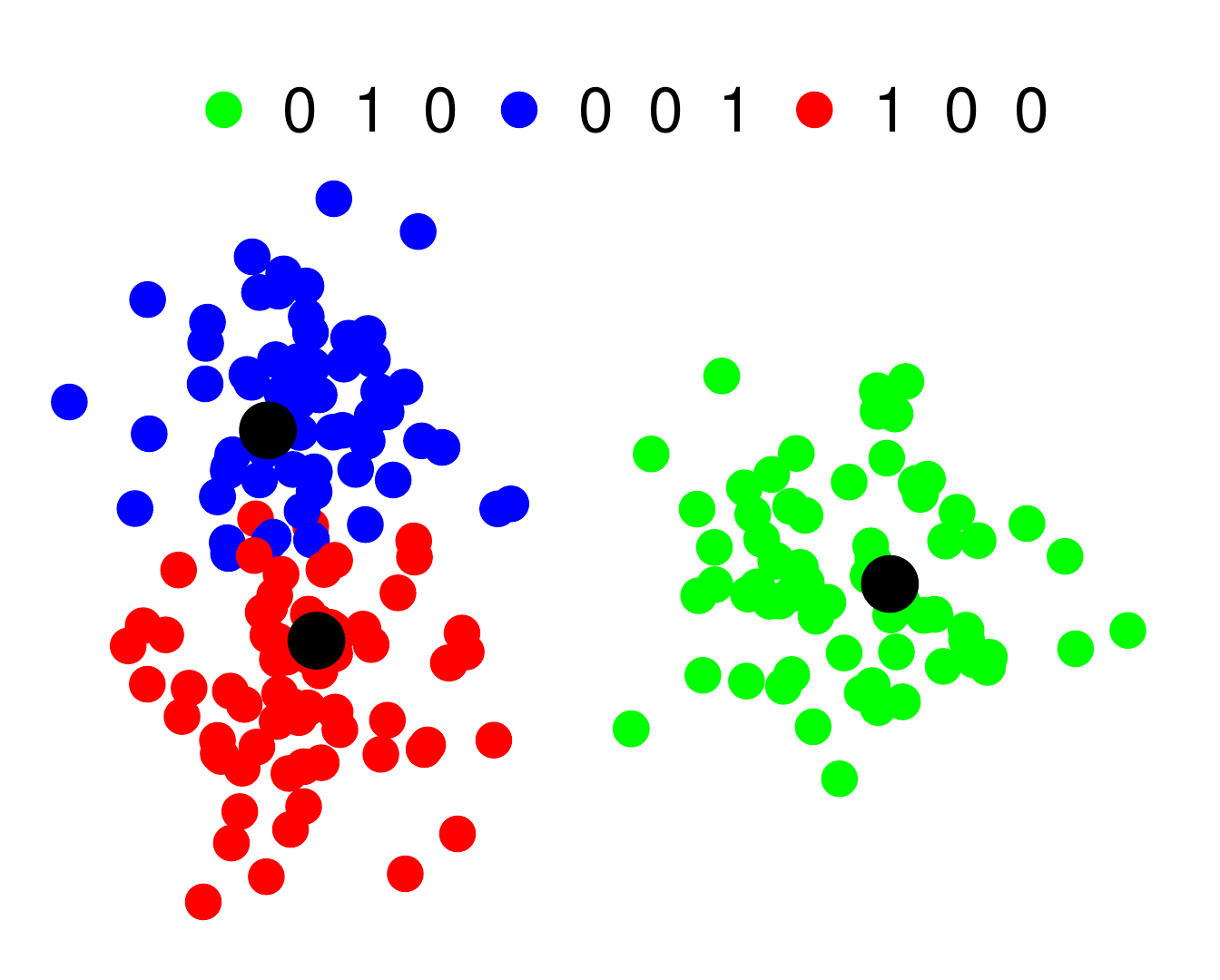}}
   \subfigure[]{
   \includegraphics[width=0.19\textwidth]{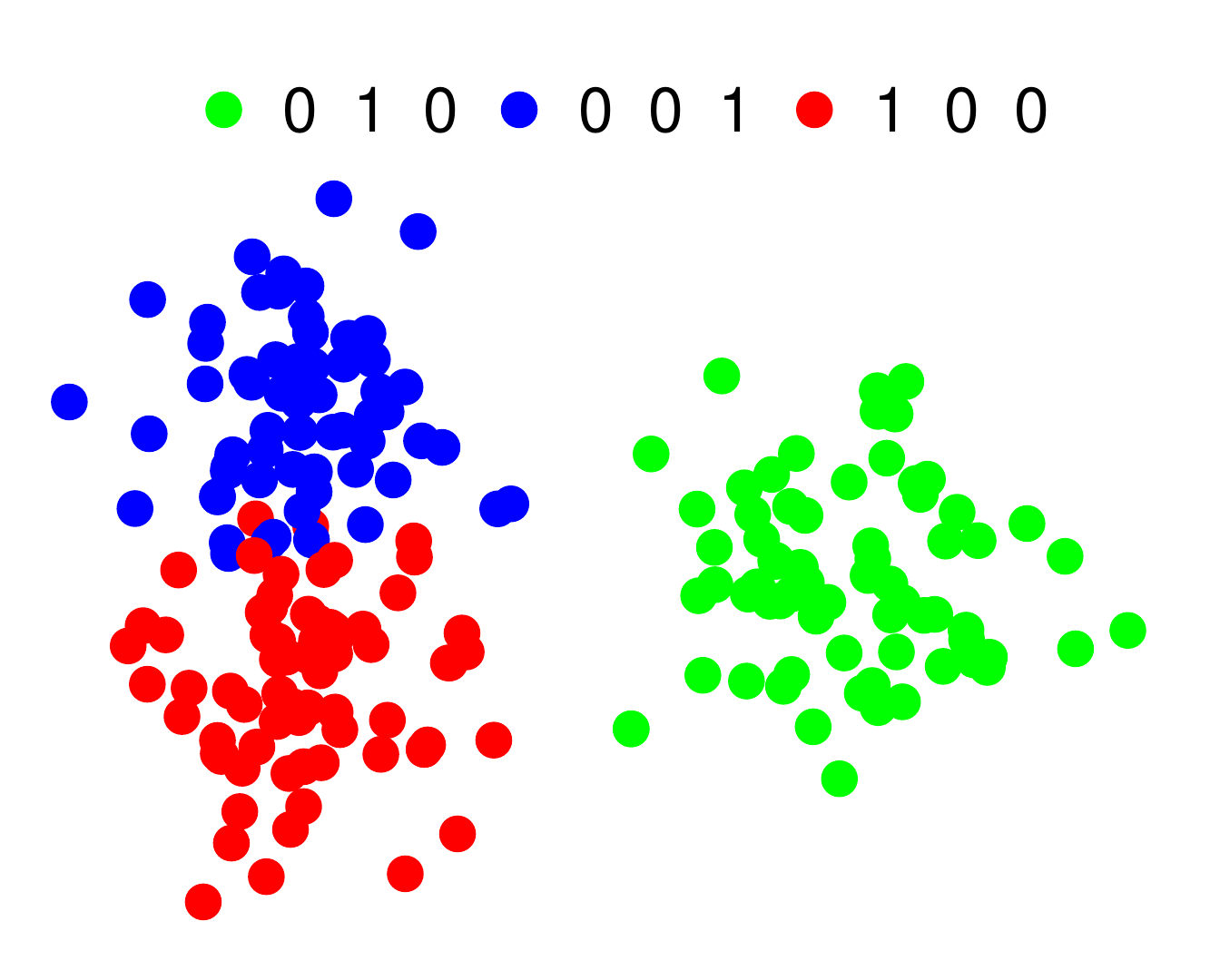}}
   \vspace{-5mm}
\caption{Toy example illustrating that the relaxation of \cite{BreLauUmiBre2013} converges to a degenerate solution when
applied to a graph with dominating $2$-cut. (a) 10NN-graph generated from three Gaussians in 10 dimensions 
(b) continuous solution of \eqref{eq:BreRel} from \cite{BreLauUmiBre2013} for $k=3$, 
(c) rounding of the continuous solution of \cite{BreLauUmiBre2013} does not yield a $3$-partition
(d) continuous solution found by our method together with the vertices $i \in I$ (black) where the membership constraint is enforced. Our continuous solution corresponds already to a partition.
(e) clustering found by rounding of our continuous solution (trivial as we have converged to a partition).
 In (b)-(e), we color data point $i$ according to $F_{(i)} \in \R^3$.}
\label{fig:counterEX}
\end{figure}

\section{Monotonic Descent Method for Minimization of a Sum of Ratios}\label{sec:alg}	
Apart from the new relaxation another key contribution of this paper is the derivation of an algorithm which yields a
sequence of feasible points for the difficult non-convex problem \eqref{eq:contRel} and reduces monotonically the corresponding objective. 
We would like to note that the algorithm proposed by \cite{BreLauUmiBre2013} for  \eqref{eq:BreRel} does not yield monotonic descent. In fact it is unclear what
the derived guarantee for the algorithm in \cite{BreLauUmiBre2013} implies for the generated sequence. Moreover, our algorithm works
for any non-negative submodular balancing function.

The key insight in order to derive a monotonic descent method for solving the sum-of-ratio minimization problem \eqref{eq:contRel}
is to eliminate the ratio by introducing a new set of variables $\beta = (\beta_1, \ldots, \beta_k)$.
\begin{alignat}{3}\label{eq:contNonCvx}
	\min_{\substack{F = (F_1, \ldots, F_k),\\ F\in \R_+^{n \times k},\ \beta \in \R^k_+ }} &\; \sum_{l=1}^k \beta_l \\
	\subj &\; \TV(F_l) \le \beta_l S(F_l),\quad &&  l = 1, \ldots, k, \nonumber && \textrm{(\descentCnstrs)}\\
			 &\; F_{(i)} \in \Delta_k,\quad &&  i = 1, \ldots, n,\quad && \textrm{(\simplexCnstrs)}\nonumber \\
			 &\; \max\{F_{(i)}\} = 1,\quad && \forall i \in I, \quad && \textrm{(\membershipCnstrs)} \nonumber\\
			 &\; S(F_l) \ge m,\quad && l = 1, \ldots, k.\quad && \textrm{(\sizeCnstrs)} \nonumber
\end{alignat}	
Note that for the optimal solution $(F^\opt, \beta^\opt)$ of this problem it holds $\TV(F^\opt_l) = \beta^\opt_l S(F^\opt_l), l = 1, \ldots, k$ (otherwise one can decrease $\beta^\opt_l$ and hence the objective) and thus equivalence holds. This is still a non-convex problem as the descent, membership and
size constraints are non-convex. Our algorithm proceeds now in a sequential manner. At each iterate we do a convex inner approximation of the constraint
set, that is the convex approximation is a subset of the non-convex constraint set, based on the current iterate $(F^t, \beta^t)$. Then we optimize the resulting convex optimization problem and repeat the process. In this way we get a sequence of feasible points for the original problem \eqref{eq:contNonCvx}
for which we will prove monotonic descent in the sum-of-ratios.

\textbf{Convex approximation:} As $\hat{S}$ is submodular, $S$ is convex. Let $s^t_l \in \partial S(F_l^t)$ be an element of the sub-differential of $S$ at the current iterate $F^t_l$. 
We have by Prop. 3.2 in \citep{Bac2013}, $(s^t_l)_{j_i} = \hat{S}(C_{l_{i-1}}) - \hat{S}(C_{l_i})$, where $j_i$ is the index of the $i^{th}$ smallest component of $F^t_l$ and $C_{l_i} = \{ j  \in V \, | \, (F^t_l)_j > (F^t_l)_i\}$.
Moreover, using the definition of subgradient, we have $S(F_l) \ge S(F^t_l) + \inner{s^t_l, F_l - F_l^t} = \inner{s_l^t, F_l}$.

For the \descentCnstrs, let $\lambda_l^t = \frac{\TV(F^t_l)}{S(F^t_l)}$ and introduce new variables $\delta_l = \beta_l - \lambda_l^t$ that capture the amount of change in each ratio. 
We further decompose $\delta_l$ as $\delta_l = \delta_l^+ - \delta_l^-,\ \delta_l^+ \ge 0,\ \delta_l^- \ge 0$.
Let $M = \max_{f \in [0,1]^n} S(f)=\max_{C \subset V} \hat{S}(C)$, then for $S(F_l)\ge m$,
\begin{align*}
 \TV(F_l) - \beta_l S(F_l) & 
 \le  \TV(F_l) - \lambda^t_l \inner{s_l^t, F_l} - \delta_l^+ S(F_l) + \delta_l^- S(F_l) \\
 &\le \TV(F_l) - \lambda^t_l \inner{s_l^t, F_l} - \delta_l^+ m + \delta_l^- M
\end{align*}
Finally, note that because of the simplex constraints, the \membershipCnstrs\ can be rewritten as $\max\{F_{(i)}\} \ge 1$. 
Let $i \in I$ and define $j_i := \argmax\nolimits _{j} F^t_{ij}$ (ties are broken randomly).
Then the \membershipCnstrs\ can be relaxed as follows:
$ 0 \ge 1-\max\{F_{(i)}\} \ge 1- F_{i j_i}  \implies  F_{i j_i} \ge 1$.
As $F_{ij} \le 1$ we get $F_{ij_i} = 1$. Thus the convex approximation of the membership constraints fixes the assignment
of the $i$-th point to a cluster and thus can be interpreted as ``label constraint''.
However, unlike the transductive setting, the labels for the vertices in $I$ are automatically chosen by our method.
The actual choice of the set $I$ will be discussed in Section \ref{sec:Implementation}.
We use the notation $L = \{(i, j_i)\ |\ i \in I\}$ 
for the label set generated from $I$ (note that $L$ is fixed once $I$ is fixed).

\textbf{Descent algorithm:}
Our descent algorithm for minimizing  \eqref{eq:contNonCvx} 
solves at each iteration $t$ the following convex optimization problem \eqref{eq:innerCnstr}.
\begin{align}\label{eq:innerCnstr}
	\min_{\substack{F \in \R_+^{n \times k},\\ \delta^+ \in \R_+^k,\ \delta^- \in \R_+^k}} &\; \sum_{l=1}^k \delta_l^+ - \delta_l^-\\
		\subj &\; \TV(F_l) \le \lambda^t_l \inner{s_l^t, F_l} + \delta_l^+ m - \delta_l^- M,\ &&   l = 1, \ldots k, \nonumber && \textrm{(\descentCnstrs)}\\
			 &\; F_{(i)} \in \Delta_k,\quad &&  i = 1, \ldots, n,\quad && \textrm{(\simplexCnstrs)}\nonumber \\
			 &\; F_{ij_i} = 1,\quad && \forall (i, j_i) \in L, \quad && \textrm{(\labelCnstrs)} \nonumber\\
			 &\; \inner{s_l^t, F_l^t} \ge m,\quad &&  l = 1, \ldots, k.\quad && \textrm{(\sizeCnstrs)} \nonumber
	\end{align}	
As its solution $F^{t+1}$ is feasible for \eqref{eq:contRel} we update $\lambda^{t+1}_l=\frac{\TV(F_l^{t+1})}{S(F_l^{t+1})}$ and $s^{t+1}_l \in \partial S(F^{t+1}_l), \;l=1, \ldots, k$ and repeat the process until the sequence terminates, that is no further descent is possible as the following theorem states,
or the relative descent in $\sum_{l=1}^k \lambda^t_l$ is smaller than a predefined $\epsilon$.
The following Theorem \ref{th:monotone} shows the monotonic descent property of our algorithm.
	\begin{theorem}\label{th:monotone}
		The sequence $\{F^t\}$ produced by the above algorithm satisfies 
		$\sum_{l=1}^k \frac{\TV(F_l^{t+1})}{S(F_l^{t+1})} < \sum_{l=1}^k\frac{\TV(F_l^{t})}{S(F_l^{t})}$ for all $t\ge 0$ or the algorithm terminates.
	\end{theorem}	
\iflongversion
		\begin{proof}
Let $(F^{t+1}, \delta^{+,\ {t+1}}, \delta^{-,\ t+1})$ be the optimal solution of the  inner problem \eqref{eq:innerCnstr}.
By the feasibility of $(F^{t+1}, \delta^{+,\ {t+1}}, \delta^{-,\ t+1})$ and $S(F^{t+1}_l)\ge m$,
\begin{align*}
		\frac{\TV(F_l^{t+1})}{S(F_l^{t+1})}	&\le \frac{\lambda^t_l \inner{s^t_l, F_l^{t+1}} + m \delta^{+,\ {t+1}}_l - M \delta^{-,\ t+1}_l}{S(F_l^{t+1})}\\
	&\le \lambda^t_l + \frac{m \delta^{+,\ {t+1}}_l - M \delta^{-,\ {t+1}}_l}{S(F_l^{t+1})}
	\le \lambda^t_l + \delta^{+,\ {t+1}}_l -\delta^{-,\ {t+1}}_l
\end{align*}
Summing over all ratios, we have 
\begin{align*}
	\sum_{l=1}^k \frac{\TV(F_l^{t+1})}{S(F_l^{t+1})}	&\le \sum_{l=1}^k \lambda^t_l  + \sum_{l=1}^k  \delta_l^{+,\ {t+1}} - \delta_l^{-,\ {t+1}}
\end{align*}

Noting that  $\delta^+_l = \delta^-_l=0,\ F = F^t$ is feasible for \eqref{eq:innerCnstr}, the optimal value $\sum_{l=1}^k \delta_l^{+,\ {t+1}} - \delta_l^{-,\ {t+1}}$ has to be either strictly negative in which case we have strict descent
\[ \sum_{l=1}^k \frac{\TV(F_l^{t+1})}{S(F_l^{t+1})} < \sum_{l=1}^k \lambda^t_l \]
or the previous iterate $F^t$ together with $\delta^+_l = \delta^-_l=0$ is already optimal and hence the algorithm terminates.
\end{proof}
\fi
The inner problem \eqref{eq:innerCnstr} is convex, but contains the non-smooth term $\TV$ in the constraints. 
We eliminate the non-smoothness by introducing additional variables and derive an equivalent linear programming (LP) formulation. 
We solve this LP via the PDHG algorithm \cite{ChaPoc11,PocCha11}. The LP and the exact iterates can be found in the supplementary material. 
\iflongversion
\begin{lemma}\label{le:LPformulation}
The convex inner problem \eqref{eq:innerCnstr} is equivalent to the following linear optimization problem where $E$ is the set of edges of the graph and $w \in \R^{\abs{E}}$ are the edge weights.
\begin{align}\label{eq:LP}
	\min_{\substack{F \in \R_+^{n \times k},\\ \alpha \in \R_+^{\abs{E} \times k},\\ \delta^+ \in \R_+^k,\ \delta^- \in \R_+^k}} &\; \sum_{l=1}^k \delta_l^+ - \delta_l^-\\
		\subj &\; \inner{w, \alpha_l} \le \lambda^t_l \inner{s_l^t, F_l} + \delta_l^+ m - \delta_l^- M,\ &&  l = 1, \ldots, k, \nonumber && \textrm{(\descentCnstrs)}\\
			 &\; F_{(i)} \in \Delta_k,\quad && i = 1, \ldots, n,\quad && \textrm{(\simplexCnstrs)}\nonumber \\
			 			 &\; F_{ij_i} = 1,\quad && \forall (i, j_i) \in L, \quad && \textrm{(\labelCnstrs)} \nonumber\\
			 &\; \inner{s_l^t, F_l^t} \ge m,\quad && l = 1, \ldots, k,\quad && \textrm{(\sizeCnstrs)} \nonumber\\
		&\; -(\alpha_l)_{ij} \le F_{il} - F_{jl} \le (\alpha_l)_{ij}, && l = 1, \ldots, k, \quad && \forall (i, j) \in E. \nonumber					 
\end{align}		
\end{lemma}
\begin{proof}
	We define new variables $\alpha_l \in \R^{\abs{E}}$ for each column $l$ and introduce constraints $(\alpha_l)_{ij} =  \abs{(F_l)_i - (F_l)_j)}$, which allows us to rewrite $\TV(F_l)$ as $\inner{w, \alpha_l}$.
	These equality constraints can be replaced by the inequality constraints $(\alpha_l)_{ij} \ge\abs{(f_l)_i - (f_l)_j)}$ without changing the optimality of the problem, because at the optimal these constraints are active.
	Otherwise one can decrease $(\alpha_l)_{ij}$ while still being feasible since $w$ is non-negative.
	Finally, these inequality constraints are rewritten using the fact that $\abs{x} \le y \Leftrightarrow -y \le x \le y$, for $y\ge0$.
\end{proof}
\subsubsection{Solving LP via PDHG}

Recently, first-order primal-dual hybrid gradient descent (PDHG for short) methods have been proposed \cite{EssZhaCha10, ChaPoc11} to efficiently solve a class of convex optimization problems that can be rewritten as the following saddle-point problem
\begin{align*}
	\min_{x \in X} \max_{y \in Y} \inner{Ax, y} + G(x) - \Phi^*(y),
\end{align*}
where $X$ and $Y$ are finite-dimensional vector spaces and $A: X \rightarrow Y$ is a linear operator and $G$ and $\Phi^*$ are convex functions.
It has been shown that the PDHG algorithm achieves good performance in solving huge linear programming problems that appear in computer vision applications.
We now show how the linear programming problem 
\begin{align*}
	\min_{x \ge 0} &\; \inner{c, x}\\
	\subj &\;  A_1 x \le b_1\\
	&\; A_2 x = b_2	
\end{align*}
can be rewritten as a saddle-point problem so that PDHG can be applied. 

By introducing the Lagrange multipliers $y$, the optimal value of the LP can be written as\[ \min_{x \ge 0} \inner{c, x} + \max_{y_1 \ge 0,\ y_2} \inner{y_1, A_1x - b_1} + \inner{y_2, A_2x - b_2} \]
\[ = \min_{x} \max_{y_1,\ y_2} \inner{c, x} + \iota_{x \ge 0} (x) + \inner{y_1, A_1x} + \inner{y_2, A_2x}  - \inner{b_1, y_1} - \inner{b_2, y_2} - \iota_{y_1 \ge 0}(y_1),\]
where $\iota_{\cdot \ge 0}$ is the indicator function that takes a value of $0$ on the non-negative orthant and $\infty$ elsewhere.

Define $b = \left( \begin{array}{c} b_1 \\ b_2 \end{array} \right),\ A = \left(\begin{array}{c} A_1\\ A_2 \end{array} \right)$ and $y = \left( \begin{array}{c} y_1 \\ y_2 \end{array} \right)$. Then the saddle point problem corresponding to the LP is given by
\[ \min_{x} \max_{y_1,\ y_2} \inner{c, x} + \iota_{x \ge 0} (x) + \inner{y, Ax} - \inner{b, y}  - \iota_{y_1 \ge 0}(y_1).\]

The primal and dual iterates for this saddle-point problem can be obtained as
\begin{align*}
	x^{r+1} &= \max\{ 0, x^r - \tau (A^T y^r + c) \},\\
	y^{r+1}_1 &= \max\{0, y_1^r + \sigma (A_1 \bar{x}^{r+1} - b_1)\},\\
	y^{r+1}_2 &= y_2^r + \sigma (A_2 \bar{x}^{r+1} - b_2),
\end{align*}
where $\bar{x}^{r+1} = 2x^{r+1} -x^r$.
Here the primal and dual step sizes $\tau$ and $\sigma$ are chosen such that $\tau \sigma \norm{A}^2 < 1$, where $\norm{.}$ denotes the operator norm.

Instead of the global step sizes $\tau$ and $\sigma$, we use in our implementation the diagonal preconditioning matrices introduced in \cite{PocCha11} as it is shown to improve the practical performance of PDHG.
The diagonal elements of these preconditioning matrices $\taubm$ and $\sigmabm$ are given by
\[ \taubm_j = \frac{1}{\sum_{i=1}^{n_r} \abs{A_{ij}} }, \forall j \in \{1, \ldots, n_c\}, \ 
   \sigmabm_i = \frac{1}{\sum_{i=1}^{n_c} \abs{A_{ij}} }, \forall i \in \{1, \ldots, n_r\},  \]
where $n_r,\ n_c$ are the number of rows and the number of columns of the matrix $A$.

%

For completeness, we now present the explicit form of the primal and dual iterates of the preconditioned PDHG for the LP \eqref{eq:LP}.
Let $\theta \in\R^k,\ \mu\in \R^n,\ \zeta \in \R^{\abs{L}},\ \nu \in \R^k,\ \eta_l \in \R^{\abs{E}},\ \xi_l \in \R^{\abs{E}},\ \forall l \in \{1, \ldots, k\}$ be the Lagrange multipliers corresponding to the descent, simplex, label, size and the two sets of additional constraints (introduced to eliminate the non-smoothness) respectively.
Let  $B: \R^{\abs{E}} \rightarrow \R^{\abs{V}}$ be a linear mapping defined as $(B z)_i= \sum_{j: (i,j) \in E} z_{ij} - z_{ji}$ and $\ones_n \in \R^n$ denote a vector of all ones.
Then the primal iterates for the LP \eqref{eq:LP} are given by
\begin{align*}
	F^{r+1}_l &= \max\Big\{0, F^r_l - \taubm_{F,\ l}\Big((-\theta^r_l\lambda_l^t - \nu^r_l) s_l^t+  \mu^r + Z^r_l + B (\eta^r_l - \xi^r_l)\Big)\Big\},\ \forall l \in \{1, \ldots, k\},\\	
	\alpha^{r+1}_l &= \max\Big\{0, \alpha^r_l - \taubm_{\alpha,\ l} \Big(\theta^r_l w - \eta^r_l - \xi^r_l \Big)\Big\},\ \forall l \in \{1, \ldots, k\},\\
	\delta^{+,\ {r+1}} &= \max\Big\{0, \delta^{+,\ r} - \taubm_{\delta^+} \Big(-m \theta^r + \ones_k \Big) \Big\},\\
	\delta^{-,\ {r+1}} &= \max\Big\{0, \delta^{-,\ r} - \taubm_{\delta^-} \Big(M \theta^r - \ones_k \Big) \Big\},
\end{align*}
where $Z^r_l \in \R^n, l=1, \ldots, k$, are given by $(Z^r_l)_{i} = \zeta^r_{il}$,  if $(i, l) \in L$ and $0$ otherwise.
Here $\taubm_{F,\ l},\ \taubm_{\alpha,\ l},\ \taubm_{\delta^+},\ \taubm_{\delta^-}$ are the diagonal preconditioning matrices whose diagonal elements are given by
\begin{align*}
	\left( \taubm_{F,\ l} \right)_i &= \frac{1}{(1+\lambda_l^t)\abs{(s^t_l)_i}+2 d_i+\rho_{il}+1}, \ \forall i \in \{ 1, \ldots, n\},\\
	\left( \taubm_{\alpha,\ l} \right)_{ij} &= \frac{1}{w_{ij}+2}, \ \forall (i,j) \in E,\\
	\left( \taubm_{\delta^+} \right)_l &= \frac{1}{m}, \ \forall l \in \{ 1, \ldots, k\},\\
	\left( \taubm_{\delta^-} \right)_l &= \frac{1}{M}, \ \forall l \in \{ 1, \ldots, k\},
\end{align*}
where $d_i$ is the number of vertices adjacent to the $i^{th}$ vertex and $\rho_{il} = 1$, if $(i,l) \in L$ and $0$ otherwise. 

The dual iterates are given by
\begin{align*}
	\theta^{r+1}_l &= \max\Big\{0, \theta^r_l + \sigmabm_{\theta,\ l} \Big(\inner{w, \bar{\alpha}^{r+1}_l} - \lambda^t_l \inner{s_l^t, \bar{F}^{r+1}_l} - m \bar{\delta}_l^{+,\ {r+1}} + M \bar{\delta}^{-,\ {r+1}}_l \Big) \Big\},\ l = 1, \ldots, k,\\
	\mu^{r+1} &=  \mu^r + \sigmabm_{\mu} \Big(\bar{F}^{r+1} \ones_k - \ones_n\Big),\\ 
	\zeta^{r+1}_{il} &= \zeta^r_{il} + \sigmabm_{\zeta} \Big(\bar{F}^{r+1}_{il} - 1 \Big), \  \forall (i, l) \in L,\\
	\nu^{r+1}_l &= \max\Big\{0, \nu^r_l + \sigmabm_{\nu,\ l} \Big(-\inner{s_l^t, \bar{F}^{r+1}_l} + m\Big)\Big\},\ \forall l \in \{1, \ldots, k\},\\
	\eta^{r+1}_l &= \max\Big\{0, \eta^r_l + \sigmabm_{\eta,\ l} \Big( -\bar{\alpha}^{r+1}_l + \bar{F}^{r+1}_{il} - \bar{F}^{r+1}_{jl}\Big)\Big\},\ \forall l \in \{1, \ldots, k\},\\
	\xi^{r+1}_l &= \max\Big\{0, \xi^r_l + \sigmabm_{\xi,\ l} \Big( -\bar{\alpha}^{r+1}_l - \bar{F}^{r+1}_{il} + \bar{F}^{r+1}_{jl}\Big)\Big\},\ \forall l \in \{1, \ldots, k\},	
\end{align*}
where 
\begin{align*}
	\sigmabm_{\theta,\ l} = \frac{1}{\inner{w,1} + \lambda^t_l\sum_{i=1}^n \abs{(s^t_l)_i} + m +M},\ 
	\sigmabm_{\zeta} = 1,\ 
	\sigmabm_{\nu,\ l}= \frac{1}{\sum_{i=1}^n \abs{(s^t_{l})_i}},
\end{align*}
and $\sigmabm_{\mu},\ \sigmabm_{\eta,1},\ \sigmabm_{\xi,l}$ are the diagonal preconditioning matrices whose diagonal elements are given by
\begin{align*}
	\left( \sigmabm_{\mu} \right)_i = \frac{1}{k},\ \forall i \in \{1, \ldots, n\},\ 
	\left(\sigmabm_{\eta,l} \right)_{ij} = \left(\sigmabm_{\xi,l} \right)_{ij} = \frac{1}{3},\ \forall (i,j) \in E.
\end{align*}

From the iterates, one sees that the computational cost per iteration is $O(\abs{E})$.
In our implementation, we further reformulated the LP \eqref{eq:LP} by directly integrating the label constraints, thereby reducing the problem size and getting rid of the dual variable $\zeta$.
\fi

\subsection{Choice of membership constraints $I$ }\label{sec:Implementation}
The overall algorithm scheme for solving the problem \eqref{eq:setProb} is given in the supplementary material.
For the membership constraints we start initially with $I^0 = \emptyset$ and sequentially solve the inner problem \eqref{eq:innerCnstr}. From its solution $F^{t+1}$ we construct a $P'_k=(C_1, \ldots, C_k)$ via rounding, see \eqref{eq:rounding}. We repeat this process until we either do not improve the resulting
balanced $k$-cut or $P'_k$ is not a partition. In this case we update $I^{t+1}$ and double the number of membership constraints. Let $(C_1^*, \ldots, C_k^*)$ be the current best partition. For each $l \in \{1, \ldots, k\}$ and $i \in C^*_l$ we compute \vspace{-1.1mm}
\begin{align}\label{eq:vertexOrdering}
b^*_{li} = \frac{\cut\big(C^*_l \backslash \{i\},\,\overline{C^*_l} \cup \{i\}\big)}{\hat{S}(C^*_l \backslash \{i\})}
+ \min_{s \neq l} \left[ \frac{\cut\big(C^*_s \cup \{i\},\,\overline{C^*_s} \backslash \{i\}\big)}{\hat{S}(C^*_s \cup \{i\})} + \sum_{j \neq l,\ j \neq s} \frac{\cut(C_j, \overline{C_j})}{\hat{S}(C_j)} \right]
\end{align}
\vspace{-1.1mm}
and define  $\Oc_l= \{ (\pi_1, \ldots, \pi_{|C^*_l|}) \,|\, b^*_{l\pi_1} \geq  b^*_{l\pi_2}\geq \ldots \geq b^*_{l\pi_{|C^*_l|}}\}.$
The top-ranked vertices in $\Oc_l$ correspond to the ones which lead to the largest minimal increase in $\BCut$ when moved from $C^*_l$ to another component and thus are most likely to belong to their current component. Thus it is natural to fix the top-ranked vertices for each component first.
Note that the rankings $\Oc_l,\,l=1, \ldots, k$  are updated when a better partition is found. Thus the \membershipCnstrs\ correspond always to the vertices
which lead to largest minimal increase in $\BCut$ when moved to another component. In Figure \ref{fig:counterEX} one can observe that the fixed labeled
points are lying close to the centers of the found clusters. The number of membership constraints depends on the graph. The better separated the clusters
are, the less membership constraints need to be enforced in order to avoid degenerate solutions. 
Finally, we stop the algorithm if we see no more improvement in the cut or the continuous objective and the continuous solution corresponds to a partition. 

\iflongversion
\begin{algorithm}[htb]
   \label{alg:overallMethod}
   \caption{{\bf for solving \eqref{eq:setProb}}}
\begin{algorithmic}[1]
   \STATE {\bfseries Initialization:} $F^0 \in \R^{n\times k}_+$
   be such that $F^0 \ones_k = \ones_n$,\   
   $\lambda^0_l = \frac{\TV(F_l^0)}{S(F_l^0)},\ l = 1, \dots, k$, $\gamma^0 = \sum_{l=1}^k \lambda^0_l$,\  $I^0 = \emptyset$,\  $L = \emptyset$,\, $p = 0$\\
   \STATE {\bfseries Output:} partition $(C^{*}_1, \ldots, C^{*}_k)$
   \REPEAT
   \STATE  $(F^{t+1}, \delta^{+,\ {t+1}}, \delta^{-,\ {t+1}})$ be the optimal solution of the inner problem \eqref{eq:innerCnstr}
   \STATE $\lambda^{t+1}_l = \frac{\TV(F_l^{t+1})}{S(F_l^{t+1})},\ l = 1, \ldots, k$, $\gamma^{t+1}= \sum_{l=1}^k \lambda^{t+1}_l$,
   \STATE $\chi^{t+1} =\sum_{l=1}^k \frac{\cut(C_l^{t+1}, \overline{C_l^{t+1}})}{\hat{S}(C_l^{t+1})} $, where $(C^{t+1}_1, \ldots, C^{t+1}_k)$ is obtained from $F^{t+1}$ via rounding
   \IF{$\chi^{t+1} < \chi^t$ and $(C^{t+1}_1, \ldots, C^{t+1}_k)$ is a $k$-partition}
   		\STATE $(C^{*}_1, \ldots, C^{*}_k) = (C^{t+1}_1, \ldots, C^{t+1}_k)$ 
	   \STATE compute new ordering $\Oc_l, \forall l=1, \ldots, k$ for $(C^{*}_1, \ldots, C^{*}_k)$ according to \eqref{eq:vertexOrdering}	  
	   \STATE $I^{t+1} = \bigcup_{l=1}^k \Oc^p_l$, where $\Oc^p_l$ denotes $p$ top-ranked vertices in $\Oc_l$
   	   \STATE $L = \{ (i, j_i) \ |\ i \in I^{t+1},\ j_i = \argmax_j F^{t+1}_{ij})\} $
   \ELSE 
   	 \STATE $p = \max\{2\abs{I^t},1\}$ (double the number of membership constraints)
	   \STATE $I^{t+1} = \bigcup_{l=1}^k \Oc^p_l$, where $\Oc^p_l$ denotes $p$ top-ranked vertices in $\Oc_l$
	   \STATE $L = \{ (i, j_i) \ |\ i \in I^{t+1},\ j_i = \argmax_j F^{t}_{ij})\} $
   	   \STATE $F^{t+1}=F^t$,\ $F^{t+1}_{ij} =0,\ \forall i \in I^{t+1},\ \forall j \in \{1, \ldots, k\}$,\ $F^{t+1}_{i j_i} = 1,\ \forall (i,j_i) \in L$
	   \STATE $\lambda^{t+1}_l=\frac{\TV(F_l^{t+1})} {S(F_l^{t+1})}$, $l=1,\ldots,k$
   \ENDIF
   	\UNTIL{$\chi^{t+1}= \sum_{l=1}^{k} \lambda^{t+1}_l$ and $\gamma^{t+1}=\gamma^t$}
\end{algorithmic}
\end{algorithm} 
\fi
	\section{Experiments}	
	\par We evaluate our method against a diverse selection of state-of-the-art clustering methods
like spectral clustering (Spec) \citep{Lux07}, BSpec \citep{HeiSet2011}, Graclus\footnote{\scriptsize Since \citep{DhiGuaKul2007}, a multi-level algorithm directly minimizing Rcut/Ncut, is shown to be superior to METIS \citep{KarKum98}, we do not compare with \citep{KarKum98}.}
\citep{DhiGuaKul2007}, NMF based approaches PNMF \citep{YanOja10}, NSC \citep{DingLi08}, ONMF \citep{DinLi06}, LSD \citep{AroGup11}, NMFR \citep{YanHao12} and MTV \citep{BreLauUmiBre2013} which optimizes \eqref{eq:BreRel}.
We used the publicly available code \cite{YanHao12, BreLauUmiBre2013} with default settings.
We run our method using 5 random initializations, 7 initializations based on the spectral clustering solution similar to \cite{BreLauUmiBre2013} (who use 30 such initializations). 
In addition to the datasets provided in \cite{BreLauUmiBre2013}, we also selected a variety of datasets from the UCI repository shown below. 
For all the datasets not in \cite{BreLauUmiBre2013}, symmetric $k$-NN graphs are built with Gaussian weights  $\exp\big({-\frac{s\norm{x-y}^2}{\min\{\sigma^2_{x,k}, \sigma^2_{y,k}\}}}\big)$, where $\sigma_{x,k}$ is the $k$-NN distance of point $x$.
We chose the parameters $s$ and $k$ in a \textit{method independent way} by testing for each dataset several graphs using all the methods over different choices of $k \in\{3,5,7,10,15,20,40,60,80,100\}$ and $s\in\{0.1,1,4\}$. 
The best choice in terms of the clustering error across all the methods and datasets, is $s=1,k=15$.	\vspace{-4mm} \setlength{\tabcolsep}{6pt}
		\begin{center}{\scriptsize
				\begin{tabular}{cccccccccccc}
					& Iris & wine & vertebral & ecoli & 4moons & webkb4 & optdigits & USPS & pendigits & 20news & MNIST\\
					\midrule
					\# vertices & 150 & 178 & 310 & 336 & 4000 & 4196& 5620 & 9298 & 10992 & 19928 & 70000 \\
					\# classes & 3 & 3 & 3 & 6 & 4 & 4 & 10 & 10 & 10 & 20  & 10\\
				\end{tabular}
			}	\end{center}
\vspace{-1mm}
	\textbf{Quantitative results:}
	In our first experiment we evaluate our method in terms of solving the balanced $k$-cut problem for various balancing functions, data sets and graph parameters. 
	The following table reports the fraction of times a method achieves the best as well as strictly best balanced $k$-cut over all constructed graphs 
	and datasets (in total $30$ graphs per dataset).
	For reference, we also report the obtained cuts for other clustering methods although they do not directly minimize this criterion in \textit{italic}; methods that directly optimize the criterion are shown in normal font.
	Our algorithm can handle all balancing functions and significantly outperforms all other methods across all criteria.
	For ratio and normalized cut cases we achieve better results than \cite{Lux07,HeiSet2011,DhiGuaKul2007} which directly optimize this criterion. 
	This shows that the greedy recursive bi-partitioning affects badly the performance of \cite{HeiSet2011}, which, otherwise, was shown to obtain the best cuts on several benchmark datasets \citep{SopWal04}. This further shows the need for methods that directly minimize the multi-cut.	
	It is striking that the competing method of \cite{BreLauUmiBre2013}, which directly minimizes the asymmetric ratio cut, is beaten significantly by Graclus
	as well as our method. As this clear trend is less visible in the qualitative experiments, we suspect that extreme graph parameters lead to fast convergence
	to a degenerate solution.
	\setlength{\tabcolsep}{3pt}
			\begin{center}{\scriptsize
				\begin{tabular}{ccccccc|ccccc}
					\toprule 
					\hspace{13mm} & \hspace{20mm} &	\footnotesize {\hspace{1mm}Ours\hspace{1mm}} & {\hspace{1mm}\footnotesize MTV\hspace{1mm}} & {\footnotesize BSpec }& {\footnotesize Spec }& {\hspace{-1mm}\footnotesize Graclus \hspace{-1mm}} &{\footnotesize  PNMF }&{\footnotesize  NSC } & {\hspace{-0.5mm}\footnotesize ONMF\hspace{-0.5mm}}
					& {\footnotesize LSD }&{\footnotesize  NMFR }\\
					\midrule 
					\multirow{2}{*}{RCC-asym} & Best (\%) & \textbf{80.54} & 25.50 & \textit{23.49} & \textit{7.38} & \textit{38.26}  & \textit{2.01}  & \textit{5.37}  & \textit{2.01}  
					 & \textit{4.03}  & \textit{1.34}\\
					& {\scriptsize Strictly Best (\%)} &{\scriptsize \textbf{44.97}}  &{\scriptsize 10.74}   &{\scriptsize \textit{1.34}}   &{\scriptsize \textit{0.00}}   &{\scriptsize \textit{4.70}}   &{\scriptsize \textit{0.00}}   &{\scriptsize \textit{0.00}}   &{\scriptsize \textit{0.00}}   
					&{\scriptsize \textit{0.00}}   &{\scriptsize \textit{0.00}}\\ 
					\midrule
					\multirow{2}{*}{RCC-sym} & Best (\%) & \textbf{94.63}  & \textit{8.72} & \textit{19.46}  & \textit{6.71} & \textit{37.58}  & \textit{0.67}  & \textit{4.03}  & \textit{0.00}  
					& \textit{0.67}  & \textit{0.67}\\
					& {\scriptsize Strictly Best (\%)} &{\scriptsize \textbf{61.74}}   &{\scriptsize \textit{0.00}}   &{\scriptsize \textit{0.67}}   &{\scriptsize \textit{0.00}}   &{\scriptsize \textit{4.70}}   &{\scriptsize \textit{0.00}}   &{\scriptsize \textit{0.00}}   &{\scriptsize \textit{0.00}}   
					 &{\scriptsize \textit{0.00}}   &{\scriptsize \textit{0.00}}\\ 
					\midrule
					\multirow{2}{*}{NCC-asym} & Best (\%) & \textbf{93.29} & \textit{13.42} & \textit{20.13} & \textit{10.07} & \textit{38.26}  & \textit{0.67}  & \textit{5.37}  & \textit{2.01}  
					 & \textit{4.70}  & \textit{2.01}\\
					& {\scriptsize Strictly Best (\%)} &{\scriptsize \textbf{56.38}}   &{\scriptsize \textit{2.01}}   &{\scriptsize \textit{0.00}}   &{\scriptsize \textit{0.00}}   &{\scriptsize \textit{2.01}}   &{\scriptsize \textit{0.00}}   &{\scriptsize \textit{0.00}}   &{\scriptsize \textit{0.67}}   
					 &{\scriptsize \textit{0.00}}   &{\scriptsize \textit{1.34}}\\ 
					\midrule
					\multirow{2}{*}{NCC-sym} & Best (\%) & \textbf{98.66} & \textit{10.07} & \textit{20.81}  & \textit{9.40} & \textit{40.27}  & \textit{1.34}  & \textit{4.03}  & \textit{0.67} 
					& \textit{3.36}  & \textit{1.34}\\
					& {\scriptsize Strictly Best (\%)} &{\scriptsize \textbf{59.06}}   &{\scriptsize \textit{0.00}}   &{\scriptsize \textit{0.00}}   &{\scriptsize \textit{0.00}}   &{\scriptsize \textit{1.34}}   &{\scriptsize \textit{0.00}}   &{\scriptsize \textit{0.00}}   &{\scriptsize \textit{0.00}}   
					  &{\scriptsize \textit{0.00}}   &{\scriptsize \textit{0.00}}\\ 
					\midrule
					\multirow{2}{*}{Rcut} & Best (\%) &\textbf{85.91}  & \textit{7.38} & 20.13 & 10.07& 32.89  & \textit{0.67}  & \textit{4.03}  & \textit{0.00} 
					& \textit{1.34}  & \textit{1.34}\\
					& {\scriptsize Strictly Best (\%)} &{\scriptsize \textbf{58.39}}   &{\scriptsize \textit{0.00}}   &{\scriptsize 2.68}   &{\scriptsize 2.01}   &{\scriptsize 8.72}   &{\scriptsize \textit{0.00}}   &{\scriptsize \textit{0.00}}   &{\scriptsize \textit{0.00}}   
					 &{\scriptsize \textit{0.00}}   &{\scriptsize \textit{0.67}}\\ 
					\midrule
					\multirow{2}{*}{Ncut} & Best (\%) & \textbf{95.97} & \textit{10.07} & 20.13  & 9.40 & 37.58  & \textit{1.34}  & \textit{4.70}  & \textit{0.67}  
					  & \textit{3.36}  & \textit{0.67}\\
					& {\scriptsize Strictly Best (\%)} &{\scriptsize \textbf{61.07}}   &{\scriptsize \textit{0.00}}   &{\scriptsize 0.00}   &{\scriptsize 0.00}   &{\scriptsize \textit4.03}   &{\scriptsize \textit{0.00}}   &{\scriptsize \textit{0.00}}   &{\scriptsize \textit{0.00}}  
					 &{\scriptsize \textit{0.00}}   &{\scriptsize \textit{0.00}}\\ 
					\bottomrule
				\end{tabular}
			}	\end{center}
	\textbf{Qualitative results:}
	In the following table, we report the clustering errors and the balanced $k$-cuts obtained by all methods using the graphs built with $k=15,\ s=1$ for all datasets. As the main goal is to compare to \cite{BreLauUmiBre2013} we choose their balancing function (RCC-asym).
	Again, our method always achieved the best cuts across all datasets. 
In three cases, the best cut also corresponds to the best clustering performance. 
In case of vertebral, 20news, and webkb4 the best cuts actually result in high errors.
However, we see in our next experiment that integrating ground-truth label information helps in these cases to improve the clustering performance significantly.

\setlength{\tabcolsep}{4pt}
			{\scriptsize
					\begin{center}
						\begin{tabular}[!t]{ccccccccccccc}
							\toprule
							\hspace{8mm} & \hspace{8mm} & {\hspace{2mm}Iris\hspace{2mm}} & {\hspace{1mm}wine\hspace{1mm}} & {vertebral} & {\hspace{1mm}ecoli\hspace{1mm}} & {4moons} & {webkb4} & {optdigits} & \scriptsize{\hspace{0.7mm}USPS\hspace{0.7mm}} & {\hspace{-0.7mm}pendigits\hspace{-0.7mm}} & 20news & \scriptsize{MNIST}\\
							\midrule 
							\multirow{2}{*}{BSpec} & \scriptsize{Err(\%)} & 23.33 &37.64 &50.00 &19.35 &36.33 &60.46 &11.30 &20.09 &17.59 &84.21 &11.82 \\
							& \scriptsize{BCut} &{\scriptsize \textbf{1.495}} &{\scriptsize 6.417} &{\scriptsize \textbf{1.890}} &{\scriptsize 2.550} &{\scriptsize 0.634} &{\scriptsize \textbf{1.056}} &{\scriptsize 0.386} &{\scriptsize 0.822} &{\scriptsize 0.081} &{\scriptsize 0.966} &{\scriptsize 0.471} \\
							\midrule 
							\multirow{2}{*}{Spec} & \scriptsize{Err(\%)} &\textbf{22.00} &20.22 &48.71 &\textbf{14.88} &31.45 &60.32 &7.81 &21.05 &16.75 &79.10 &22.83 \\
							& \scriptsize{BCut} &{\scriptsize 1.783} &{\scriptsize 5.820} &{\scriptsize 1.950} &{\scriptsize 2.759} &{\scriptsize 0.917} &{\scriptsize 1.520} &{\scriptsize 0.442} &{\scriptsize 0.873} &{\scriptsize 0.141} &{\scriptsize 1.170} &{\scriptsize 0.707} \\
							\midrule 
							\multirow{2}{*}{PNMF} & \scriptsize{Err(\%)} &22.67 &27.53 &50.00 &16.37 &35.23 &60.94 &10.37 &24.07 &17.93 &66.00 &12.80 \\
							& \scriptsize{BCut} &{\scriptsize 1.508} &{\scriptsize 4.916} &{\scriptsize 2.250} &{\scriptsize 2.652} &{\scriptsize 0.737} &{\scriptsize 3.520} &{\scriptsize 0.548} &{\scriptsize 1.180} &{\scriptsize 0.415} &{\scriptsize 2.924} &{\scriptsize 0.934} \\
							\midrule 
							\multirow{2}{*}{NSC} & \scriptsize{Err(\%)} &23.33 &17.98 &50.00 &\textbf{14.88} &32.05 &59.49 &8.24 &20.53 &19.81 &78.86 &21.27 \\
							& \scriptsize{BCut} &{\scriptsize 1.518} &{\scriptsize 5.140} &{\scriptsize 2.046} &{\scriptsize 2.754} &{\scriptsize 0.933} &{\scriptsize 3.566} &{\scriptsize 0.482} &{\scriptsize 0.850} &{\scriptsize 0.101} &{\scriptsize 2.233} &{\scriptsize 0.688} \\
							\midrule 
							\multirow{2}{*}{ONMF} & \scriptsize{Err(\%)} &23.33 &28.09 &50.65 &16.07 &35.35 &60.94 &10.37 &24.14 &22.82 &69.02 &27.27 \\
							& \scriptsize{BCut} &{\scriptsize 1.518} &{\scriptsize 4.881} &{\scriptsize 2.371} &{\scriptsize 2.633} &{\scriptsize 0.725} &{\scriptsize 3.621} &{\scriptsize 0.548} &{\scriptsize 1.183} &{\scriptsize 0.548} &{\scriptsize 3.058} &{\scriptsize 1.575} \\
							\midrule 
							\multirow{2}{*}{LSD} & \scriptsize{Err(\%)} &23.33 &17.98 &39.03 &18.45 &35.68 &47.93 &8.42 &22.68 &13.90 &67.81 &24.49 \\
							& \scriptsize{BCut} &{\scriptsize 1.518} &{\scriptsize 5.399} &{\scriptsize 2.557} &{\scriptsize 2.523} &{\scriptsize 0.782} &{\scriptsize 2.082} &{\scriptsize 0.483} &{\scriptsize 0.918} &{\scriptsize 0.188} &{\scriptsize 2.056} &{\scriptsize 0.959} \\
							\midrule 
							\multirow{2}{*}{NMFR} & \scriptsize{Err(\%)} &\textbf{22.00} &11.24 &38.06 &22.92 &36.33 &40.73 &2.08 &22.17 &13.13 &\textbf{39.97} & \footnotesize{fail} \\
							& \scriptsize{BCut} &{\scriptsize 1.627} &{\scriptsize 4.318} &{\scriptsize 2.713} &{\scriptsize 2.556} &{\scriptsize 0.840} &{\scriptsize 1.467} &{\scriptsize 0.369} &{\scriptsize 0.992} &{\scriptsize 0.240} &{\scriptsize 1.241} &{\scriptsize -} \\
							\midrule 
							\multirow{2}{*}{Graclus} & \scriptsize{Err(\%)} &23.33 &8.43 &49.68 &16.37 &\textbf{0.45} &\textbf{\textbf{39.97}} &\textbf{1.67} &19.75 &\textbf{10.93} &60.69 &2.43 \\
							& \scriptsize{BCut} &{\scriptsize 1.534} &{\scriptsize 4.293} &{\scriptsize \textbf{1.890}} &{\scriptsize 2.414} &{\scriptsize \textbf{0.589}} &{\scriptsize 1.581} &{\scriptsize \textbf{0.350}} &{\scriptsize 0.815} &{\scriptsize 0.092} &{\scriptsize 1.431} &{\scriptsize 0.440} \\
							\midrule 
							\multirow{2}{*}{MTV} & \scriptsize{Err(\%)} &22.67 &18.54 &\textbf{34.52} &22.02 &7.72 &48.40 &4.11 &\textbf{15.13} &20.55 &72.18 &3.77 \\
							& \scriptsize{BCut} &{\scriptsize 1.508} &{\scriptsize 5.556} &{\scriptsize 2.433} &{\scriptsize 2.500} &{\scriptsize 0.774} &{\scriptsize 2.346} &{\scriptsize 0.374} &{\scriptsize 0.940} &{\scriptsize 0.193} &{\scriptsize 3.291} &{\scriptsize 0.458} \\
							\midrule 
							\multirow{2}{*}{Ours} & \scriptsize{Err(\%)} &23.33 &\textbf{6.74} &50.00 &16.96 &\textbf{0.45} &60.46 &1.71 &19.72 &19.95 &79.51 &\textbf{2.37} \\
							& \scriptsize{BCut} &{\scriptsize \textbf{1.495}} &{\scriptsize \textbf{4.168}} &{\scriptsize \textbf{1.890}} &{\scriptsize \textbf{2.399}} &{\scriptsize \textbf{0.589}} &{\scriptsize \textbf{1.056}} &{\scriptsize \textbf{0.350}} &{\scriptsize \textbf{0.802}} &{\scriptsize \textbf{0.079}} &{\scriptsize \textbf{0.895}} &{\scriptsize \textbf{0.439}} \\
							\bottomrule 
						\end{tabular}
					\end{center}}
				\FloatBarrier
	\textbf{Transductive Setting:}
	As in \cite{BreLauUmiBre2013}, we randomly sample either one label or a fixed percentage of labels per class from the ground truth. 
	We report clustering errors and the cuts (RCC-asym) for both methods for different choices of labels. For label experiments their
	initialization strategy seems to work better as the cuts improve compared to the unlabeled case. However, observe that in some 
	cases their method seems to fail completely (Iris and 4moons for one label per class).
	\setlength{\tabcolsep}{3pt}
	{\scriptsize
			\begin{center}
				\begin{tabular}{cccccccccccccccc}
					\toprule 
					\scriptsize{Labels} & \hspace{8mm} & \hspace{8mm} & {\hspace{2mm} Iris\hspace{2mm} } & {\hspace{1mm}wine\hspace{1mm}} & {vertebral} & {\hspace{1mm}ecoli\hspace{1mm}} & {4moons} & {webkb4} & {\hspace{-0.7mm} optdigits\hspace{-0.7mm}} & \scriptsize{\hspace{0.7mm}USPS\hspace{0.7mm}} & {\hspace{-0.7mm}pendigits\hspace{-0.7mm}} & \scriptsize{20}news & \scriptsize{MNIST}\\
					\midrule 
					\multirow{4}{*}{1} & \multirow{2}{*}{MTV}& \scriptsize{Err(\%)}& 33.33& 9.55& \textbf{42.26}& \textbf{13.99}& 35.75& 51.98& \textbf{1.69}& \textbf{12.91}& 14.49& \textbf{50.96}& 2.45\\
					& & \scriptsize{BCut} & {\scriptsize 3.855}& {\scriptsize 4.288}& {\scriptsize \textbf{2.244}}& {\scriptsize \textbf{2.430}}& {\scriptsize 0.723}& {\scriptsize 1.596}& {\scriptsize \textbf{0.352}}& {\scriptsize 0.846}& {\scriptsize 0.127}& {\scriptsize 1.286}& {\scriptsize \textbf{0.439}}\\
					& \multirow{2}{*}{Ours}& \scriptsize{Err(\%)}& \textbf{22.67}& \textbf{8.99}& 50.32& 15.48& \textbf{0.57}& \textbf{45.11}& \textbf{1.69}& 12.98& \textbf{10.98}& 68.53& \textbf{2.36}\\
					&  & \scriptsize{BCut} & {\scriptsize \textbf{1.571}}& {\scriptsize \textbf{4.234}}& {\scriptsize 2.265}& {\scriptsize 2.432}& {\scriptsize \textbf{0.610}}& {\scriptsize \textbf{1.471}}& {\scriptsize \textbf{0.352}}& {\scriptsize \textbf{0.812}}& {\scriptsize \textbf{0.113}}& {\scriptsize \textbf{1.057}}& {\scriptsize 0.439}\\
					\midrule 
					\multirow{4}{*}{1\%} & \multirow{2}{*}{MTV}& \scriptsize{Err(\%)}& 33.33& 10.67& \textbf{39.03}& 14.29& \textbf{0.45}& 48.38& \textbf{1.67}& 5.21& \textbf{7.75}& 40.18& 2.41\\
					&& \scriptsize{BCut} & {\scriptsize 3.855}& {\scriptsize 4.277}& {\scriptsize 2.300}& {\scriptsize 2.429}& {\scriptsize \textbf{0.589}}& {\scriptsize 1.584}& {\scriptsize \textbf{0.354}}& {\scriptsize \textbf{0.789}}& {\scriptsize 0.129}& {\scriptsize 1.208}& {\scriptsize 0.443}\\
					& \multirow{2}{*}{Ours}& \scriptsize{Err(\%)}& \textbf{22.67}& \textbf{6.18}& 41.29& \textbf{13.99}& \textbf{0.45}& \textbf{41.63}& \textbf{1.67}& \textbf{5.13}& \textbf{7.75}& \textbf{37.42}& \textbf{2.33}\\
					&  & \scriptsize{BCut} & {\scriptsize \textbf{1.571}}& {\scriptsize \textbf{4.220}}& {\scriptsize \textbf{2.288}}& {\scriptsize \textbf{2.419}}& {\scriptsize \textbf{0.589}}& {\scriptsize \textbf{1.462}}& {\scriptsize \textbf{0.354}}& {\scriptsize 0.789}& {\scriptsize \textbf{0.128}}& {\scriptsize \textbf{1.157}}& {\scriptsize \textbf{0.442}}\\
					\midrule 
					\multirow{4}{*}{5\%} & \multirow{2}{*}{MTV}& \scriptsize{Err(\%)}& \textbf{17.33}& 7.87& 40.65& 14.58& \textbf{0.45}& 40.09& \textbf{1.51}& \textbf{4.85}& 1.79& 31.89& \textbf{2.18}\\
					& & \scriptsize{BCut} & {\scriptsize \textbf{1.685}}& {\scriptsize 4.330}& {\scriptsize \textbf{2.701}}& {\scriptsize 2.462}& {\scriptsize \textbf{0.589}}& {\scriptsize 1.763}& {\scriptsize 0.369}& {\scriptsize 0.812}& {\scriptsize 0.188}& {\scriptsize 1.254}& {\scriptsize 0.455}\\
					& \multirow{2}{*}{Ours}& \scriptsize{Err(\%)}&  \textbf{17.33}& \textbf{6.74}& \textbf{37.10}& \textbf{13.99}& \textbf{0.45}& \textbf{38.04}& 1.53& \textbf{4.85}& \textbf{1.76}& \textbf{30.07}& 2.18\\
					&  & \scriptsize{BCut} & {\scriptsize \textbf{1.685}}& {\scriptsize \textbf{4.224}}& {\scriptsize 2.724}& {\scriptsize \textbf{2.461}}& {\scriptsize \textbf{0.589}}& {\scriptsize \textbf{1.719}}& {\scriptsize \textbf{0.369}}& {\scriptsize \textbf{0.811}}& {\scriptsize \textbf{0.188}}& {\scriptsize \textbf{1.210}}& {\scriptsize \textbf{0.455}}\\
					\midrule 
					\multirow{4}{*}{10\%} & \multirow{2}{*}{MTV}& \scriptsize{Err(\%)}& 18.67& 7.30& 39.03& 13.39& \textbf{0.38}& \textbf{40.63}& \textbf{1.41}& \textbf{4.19}& \textbf{1.24}& 27.80& 2.03\\
					& & \scriptsize{BCut} & {\scriptsize \textbf{1.954}}& {\scriptsize 4.332}& {\scriptsize 3.187}& {\scriptsize \textbf{2.776}}& {\scriptsize \textbf{0.592}}& {\scriptsize 2.057}& {\scriptsize \textbf{0.377}}& {\scriptsize 0.833}& {\scriptsize \textbf{0.197}}& {\scriptsize 1.346}& {\scriptsize 0.465}\\
					& \multirow{2}{*}{Ours}& \scriptsize{Err(\%)}& \textbf{14.67}& \textbf{6.74}& \textbf{33.87}& \textbf{13.10}& \textbf{0.38}& 41.97& \textbf{1.41}& 4.25& \textbf{1.24}& \textbf{26.55}& \textbf{2.02}\\
					&  & \scriptsize{BCut} & {\scriptsize 1.960}& {\scriptsize \textbf{4.194}}& {\scriptsize \textbf{3.134}}& {\scriptsize 2.778}& {\scriptsize \textbf{0.592}}& {\scriptsize \textbf{1.972}}& {\scriptsize \textbf{0.377}}& {\scriptsize \textbf{0.833}}& {\scriptsize \textbf{0.197}}& {\scriptsize \textbf{1.314}}& {\scriptsize \textbf{0.465}}\\
					\bottomrule 
				\end{tabular}
			\end{center}}
\FloatBarrier
\section{Conclusion}
	We presented a framework for directly minimizing the balanced $k$-cut problem based on a new continuous relaxation.
	Apart from ratio/normalized cut, our method can also handle new application-specific balancing functions.
	Moreover, in contrast to a recursive splitting approach \citep{RanHei12}, our method enables the direct integration of prior information available in form of must/cannot-link constraints, which is an interesting topic for future research.
	Finally, the monotonic descent algorithm proposed for the difficult sum-of-ratios problem is another key contribution that is of independent interest.

\vspace{-1mm}	
\textbf{Acknowledgements.}
	The authors would like to acknowledge support by the DFG excellence cluster MMCI and the ERC starting grant NOLEPRO. 
{
\small
\bibliography{literature2}

\begin{thebibliography}{10}

\bibitem{DonHof1973}
W.~E. Donath and A.~J. Hoffman.
\newblock Lower bounds for the partitioning of graphs.
\newblock {\em IBM J. Res. Develop.}, 17:420--425, 1973.

\bibitem{PotSimLio1990}
A.~Pothen, H.~D. Simon, and K.-P. Liou.
\newblock Partitioning sparse matrices with eigenvectors of graphs.
\newblock {\em SIAM J. Matrix Anal. Appl.}, 11(3):430--452, 1990.

\bibitem{HagKah91}
L.~Hagen and A.~B. Kahng.
\newblock Fast spectral methods for ratio cut partitioning and clustering.
\newblock In {\em ICCAD}, pages 10--13, 1991.

\bibitem{ShiMal2000}
J.~Shi and J.~Malik.
\newblock Normalized cuts and image segmentation.
\newblock {\em IEEE Trans. Pattern Anal. Mach. Intell.}, 22:888--905, 2000.

\bibitem{NgJorWei2001}
A.~Ng, M.~Jordan, and Y.~Weiss.
\newblock On spectral clustering: Analysis and an algorithm.
\newblock In {\em NIPS}, pages 849--856, 2001.

\bibitem{DhiGuaKul2007}
I.~Dhillon, Y.~Guan, and B.~Kulis.
\newblock Weighted graph cuts without eigenvectors: A multilevel approach.
\newblock {\em IEEE Trans. Pattern Anal. Mach. Intell.}, pages 1944--1957,
  2007.

\bibitem{Lux07}
U.~{von Luxburg}.
\newblock A tutorial on spectral clustering.
\newblock {\em Statistics and Computing}, 17:395--416, 2007.

\bibitem{GuaMil1998}
S.~Guattery and G.~Miller.
\newblock On the quality of spectral separators.
\newblock {\em SIAM J. Matrix Anal. Appl.}, 19:701--719, 1998.

\bibitem{SB10}
A.~Szlam and X.~Bresson.
\newblock Total variation and {C}heeger cuts.
\newblock In {\em ICML}, pages 1039--1046, 2010.

\bibitem{HeiBue2010}
M.~Hein and T.~B{\"u}hler.
\newblock An inverse power method for nonlinear eigenproblems with applications
  in 1-spectral clustering and sparse {PCA}.
\newblock In {\em NIPS}, pages 847--855, 2010.

\bibitem{HeiSet2011}
M.~Hein and S.~Setzer.
\newblock Beyond spectral clustering - tight relaxations of balanced graph
  cuts.
\newblock In {\em NIPS}, pages 2366--2374, 2011.

\bibitem{BreLauUmiBre2012}
X.~Bresson, T.~Laurent, D.~Uminsky, and J.~H. von Brecht.
\newblock Convergence and energy landscape for {C}heeger cut clustering.
\newblock In {\em NIPS}, pages 1394--1402, 2012.

\bibitem{BreLauUmiBre2013}
X.~Bresson, T.~Laurent, D.~Uminsky, and J.~H. von Brecht.
\newblock Multiclass total variation clustering.
\newblock In {\em NIPS}, pages 1421--1429, 2013.

\bibitem{Bac2013}
F.~Bach.
\newblock Learning with submodular functions: A convex optimization
  perspective.
\newblock {\em Foundations and Trends® in Machine Learning}, 6(2-3):145--373,
  2013.

\bibitem{ChaPoc11}
A.~Chambolle and T.~Pock.
\newblock A first-order primal-dual algorithm for convex problems with
  applications to imaging.
\newblock {\em J. of Math. Imaging and Vision}, 40:120--145, 2011.

\bibitem{PocCha11}
T.~Pock and A.~Chambolle.
\newblock Diagonal preconditioning for first order primal-dual algorithms in
  convex optimization.
\newblock In {\em ICCV}, pages 1762--1769, 2011.

\bibitem{EssZhaCha10}
E.~Esser, X.~Zhang, and T.~F. Chan.
\newblock A general framework for a class of first order primal-dual algorithms
  for convex optimization in imaging science.
\newblock {\em SIAM J. on Imaging Sciences}, 3(4):1015--1046, 2010.

\bibitem{KarKum98}
G.~Karypis and V.~Kumar.
\newblock A fast and high quality multilevel scheme for partitioning irregular
  graphs.
\newblock {\em SIAM J. Sci. Comput.}, 20(1):359--392, 1998.

\bibitem{YanOja10}
Z.~Yang and E.~Oja.
\newblock Linear and nonlinear projective nonnegative matrix factorization.
\newblock {\em IEEE Transactions on Neural Networks}, 21(5):734--749, 2010.

\bibitem{DingLi08}
C.~Ding, T.~Li, and M.~I. Jordan.
\newblock Nonnegative matrix factorization for combinatorial optimization:
  Spectral clustering, graph matching, and clique finding.
\newblock In {\em ICDM}, pages 183--192, 2008.

\bibitem{DinLi06}
C.~Ding, T.~Li, W.~Peng, and H.~Park.
\newblock Orthogonal nonnegative matrix tri-factorizations for clustering.
\newblock In {\em KDD}, pages 126--135, 2006.

\bibitem{AroGup11}
R.~Arora, M.~R. Gupta, A.~Kapila, and M.~Fazel.
\newblock Clustering by left-stochastic matrix factorization.
\newblock In {\em ICML}, pages 761--768, 2011.

\bibitem{YanHao12}
Z.~Yang, T.~Hao, O.~Dikmen, X.~Chen, and E.~Oja.
\newblock Clustering by nonnegative matrix factorization using graph random
  walk.
\newblock In {\em NIPS}, pages 1088--1096, 2012.

\bibitem{SopWal04}
A.~J. Soper, C.~Walshaw, and M.~Cross.
\newblock A combined evolutionary search and multilevel optimisation approach
  to graph-partitioning.
\newblock {\em J. of Global Optimization}, 29(2):225--241, 2004.

\bibitem{RanHei12}
S.~S. Rangapuram and M.~Hein.
\newblock Constrained 1-spectral clustering.
\newblock In {\em AISTATS}, pages 1143--1151, 2012.

\end{thebibliography}
\bibliographystyle{unsrt}
}
\end{document}